\documentclass{article}

\usepackage{microtype}
\usepackage{graphicx}
\usepackage{subfigure}
\usepackage{booktabs} % for professional tables

\usepackage{hyperref}

\newcommand{\widgraph}[2]{\includegraphics[keepaspectratio,width=#1]{#2}}
\usepackage[accepted]{icml2024}

\usepackage{amsmath}
\usepackage{amssymb}
\usepackage{mathtools}
\usepackage{amsthm}
\usepackage{multicol}
\usepackage{multirow}
\usepackage[capitalize,noabbrev]{cleveref}

\theoremstyle{plain}
\newtheorem{theorem}{Theorem}
\newtheorem{proposition}{Proposition}
\newtheorem{lemma}{Lemma}

\theoremstyle{definition}
\newtheorem{definition}{Definition}

\theoremstyle{remark}
\newtheorem{remark}{Remark}

\usepackage[normalem]{ulem}

\usepackage[textsize=tiny]{todonotes}

\icmltitlerunning{Sliced Wasserstein with Random-Path Projecting Directions}

\begin{document}

\twocolumn[
\icmltitle{Sliced Wasserstein with Random-Path Projecting Directions}



\icmlsetsymbol{equal}{*}

\begin{icmlauthorlist}
\icmlauthor{Khai Nguyen}{yyy}
\icmlauthor{Shujian Zhang}{yyy}
\icmlauthor{Tam Le}{comp}
\icmlauthor{Nhat Ho}{yyy}
\end{icmlauthorlist}

\icmlaffiliation{yyy}{Department of Statistics and Data Sciences, University of Texas at Austin, USA}
\icmlaffiliation{comp}{Department of Advanced Data Science, The Institute of Statistical Mathematics (ISM), Japan}

\icmlcorrespondingauthor{Khai Nguyen}{khainb@utexas.edu}

\icmlkeywords{Machine Learning, ICML, Optimal Transport, Sliced Wasserstein, Diffusion Models. Generative Models, Random-path, Projecting Directions.}

\vskip 0.3in
]



\printAffiliationsAndNotice{} % otherwise use the standard text.

\begin{abstract}
Slicing distribution selection has been used as an effective technique to improve the performance of parameter estimators based on minimizing sliced Wasserstein distance in applications. Previous works either utilize expensive optimization to select the slicing distribution or use slicing distributions that require expensive sampling methods. In this work, we propose an optimization-free slicing distribution that provides a fast sampling for the Monte Carlo estimation of expectation. In particular, we introduce the \textit{random-path} projecting direction (RPD) which is constructed by leveraging the normalized difference between two random vectors following the two input measures. From the RPD, we derive the random-path slicing distribution (RPSD) and two variants of sliced Wasserstein, i.e., the Random-Path Projection Sliced Wasserstein (RPSW) and the Importance Weighted Random-Path Projection Sliced Wasserstein (IWRPSW). We then discuss the topological, statistical, and computational properties of RPSW and IWRPSW. Finally, we showcase the favorable performance of RPSW and IWRPSW in gradient flow and the training of denoising diffusion generative models on images.

\end{abstract}

\section{Introduction}
\label{sec:introduction}
Utilizing the closed-form solution of optimal transport in one dimension~\cite{peyre2020computational}, the sliced Wasserstein~\cite{bonneel2015sliced} (SW) distance is computationally and statistically scalable. In particular, SW has the time complexity of $\mathcal{O}(n\log n)$ and the space complexity of $\mathcal{O}(n)$ when dealing with two probability measures that have at most $n$ supports. Moreover, the SW does not suffer from the curse of dimensionality since it has sample complexity of $\mathcal{O}(n^{-1/2})$~\cite{nadjahi2020statistical,nietert2022statistical}. As evidence for the effectiveness of the SW distance, many applications have leveraged SW as the key to improving performance e.g.,  generative models~\cite{deshpande2018generative}, domain adaptation~\cite{lee2019sliced},  point-cloud upsampling~\cite{savkin2022lidar}, clustering~\cite{kolouri2018slicedgmm}, gradient flows~\cite{bonet2022efficient}, and so on.

A key component of the SW is the slicing distribution, i.e., the distribution of the projecting direction. The slicing distribution controls the contribution of each projecting direction to the final value of the SW. A \textit{statistically} desired slicing distribution should satisfy two properties. The first property is to assign a high probability to ``informative" projecting directions. In some applications that involve weak convergence of measures, such as generative modeling~\cite{deshpande2018generative} and gradient flow~\cite{bonet2022efficient}, an informative direction can be interpreted as a discriminative direction that highlights the difference between two measures after projection. The second property is to have the maximal projecting direction in the support set, i.e., the projecting direction that can lead to the maximum projected distance. This property ensures the identity of indiscernible of the expected value. To ensure the second property, a more relaxed criterion can be considered, which involves having a continuous density on the unit-hypersphere.

Beyond statistical properties, two \textit{computational} properties are desired for the slicing distribution. Firstly, the slicing distribution should be stable to obtain, meaning that parameters of the slicing distribution can be computationally identifiable. Secondly, it should be fast and efficient to sample projecting directions from the slicing distribution. This property is essential, as Monte Carlo samples are used to approximate the intractable expectation with respect to the slicing distribution.

The conventional SW~\cite{bonneel2015sliced} utilizes the uniform distribution over the unit-hypersphere as the slicing distribution, which is a fixed, continuous, and easy-to-sample distribution. However, the uniform slicing distribution is not discriminative. Distributional sliced Wasserstein~\cite{nguyen2021distributional} (DSW) performs optimization to find the best slicing distribution belonging to a chosen family of distributions over the unit-hypersphere, aiming to maximize the expected projected distance. When the family is a collection of Dirac measures, DSW reverts to max sliced Wasserstein~\cite{deshpande2019max} (Max-SW). DSW and Max-SW are discriminative, but they are not stable, computationally expensive, and require a differentiable ground metric. This is because they involve iterative gradient-based optimization, which may not guarantee a global optimum. Energy-based sliced Wasserstein~\cite{nguyen2023energy} (EBSW) utilizes the energy-based slicing distribution, which has a density proportional to an increasing function of the projected distance. The energy-based slicing distribution is continuous, optimization-free, and discriminative. However, it requires more sophisticated sampling techniques to obtain projecting directions, such as importance sampling and Markov chain Monte Carlo (MCMC), due to the unnormalized density. Moreover, these sampling techniques lead to biased weighted estimates.

To design an effective slicing distribution, the key is to relate the \textit{continuous} slicing distribution with  the two input measures. This connection can be achieved through optimization, as demonstrated in DSW~\cite{nguyen2021distributional}, or by directly defining the slicing distribution based on the two input measures, as seen in EBSW~\cite{nguyen2023energy}. Designing an efficient slicing distribution poses a more challenging task. Interactive procedures, such as optimization (as in DSW) and sequential sampling (as in EBSW with MCMC), should be avoided. The ideal scenario is to have a slicing distribution that allows for i.i.d. sampling, making it parallelizable. Using a proposal distribution, importance sampling in EBSW could yield a fast estimate. However, importance sampling requires a careful choice of the proposal distribution for accurate estimation, especially in high dimensions. Additionally, importance sampling leads to biased estimates in the case of unnormalized densities.

To address the challenge, we propose a novel, effective, and efficient slicing distribution obtained based on a new notation of projecting direction, named \textit{random-path} projecting direction (RPD). The idea is to first create a `path' between two input measures, then project the two measures along that path and conduct the one-dimensional Wasserstein comparison. Here, a path is defined as the normalized difference between two random variables that are distributed with two input measures, respectively, with an additional random perturbation for continuity guarantee. Since the path is defined as the random difference, it captures the directions in which the two measures differ. Therefore, using the distribution of the random-paths can lead to a discriminative slicing distribution. In addition, sampling a random path is efficient since it only requires the mild assumption that two input measures are easy to sample from, a condition usually held in machine learning applications.

\textbf{Contribution.} In summary, our contributions are three-fold:

1.  We propose a new type of projecting direction, named \textit{random-path} projecting direction (RPD), which involves a random perturbation of the normalized difference between two random variables distributed under two input measures. From the RPD, we derive the random-path slicing distribution (RPSD), which is guaranteed to be continuous. Despite having intractable density, the RPSD is fast and simple to sample random projecting directions.

2. From the RPSD, we introduce two novel variants of sliced Wasserstein. The first variant is called \textit{random-path projection sliced Wasserstein} (RPSW), which replaces the uniform distribution in SW with the RPSD. The second variant is \textit{importance-weighted random-path projection sliced Wasserstein} (IWRPSW). It is defined as an expectation of a weighted average of multiple randomly projected distances with RPDs. We then establish theoretical properties of RPSW and IWRPSW, encompassing topological, statistical, and computational aspects. In greater detail, we delve into the metricity of RPSW and IWRPSW, their connections to other SW variants, their sample complexity, and their computational complexities using Monte Carlo methods.

3. We conduct a comparative analysis of the proposed RPSW and IWRPSW against existing SW variants, including SW, Max-SW, DSW, and EBSW. Our first application involves gradient flow, where distances are utilized to guide a source distribution towards a target distribution. Furthermore, we introduce a new framework for training denoising diffusion models using SW metrics through the augmented generalized mini-batch energy (AGME) distance. In this framework, SW variants serve as the kernel to assess the difference between two random sets. Here, AGME functions as the loss function, aiming to minimize the difference between forward transition distributions and their corresponding reverse transition distributions in denoising diffusion models. In addition to serving as a benchmark for comparing SW variants, our proposed training approach can improve generative quality and sampling speed.

\textbf{Organization.} We begin by reviewing some preliminaries in Section~\ref{sec:preliminaries}. Following that, in Section~\ref{sec:RPSW}, we define the random-path projection direction, the random-path slicing distribution, RPSW, and IWRPSW. We then derive their theoretical and computational properties. Section~\ref{sec:experiments} presents experiments comparing RPSW and IWRPSW with other SW variants. In Section~\ref{sec:conclusion}, we provide concluding remarks. Finally, we defer the proofs of key results and additional materials to the Appendices.

\textbf{Notations.}  For any $d \geq 2$, we denote $\mathbb{S}^{d-1}:=\{\theta \in \mathbb{R}^{d}\mid  ||\theta||_2^2 =1\}$ and $\mathcal{U}(\mathbb{S}^{d-1})$ as  the unit hyper-sphere and its corresponding  uniform distribution .  We denote $\mathcal{P}(\mathcal{X})$ as the set of all probability measures on the set $\mathcal{X}$. For $p\geq 1$, $\mathcal{P}_p(\mathcal{X})$ is the set of all probability measures on the set $\mathcal{X}$ that have finite $p$-moments.  For any two sequences $a_{n}$ and $b_{n}$, the notation $a_{n} = \mathcal{O}(b_{n})$ means that $a_{n} \leq C b_{n}$ for all $n \geq 1$, where $C$ is some universal constant.  
\section{Preliminaries}
\label{sec:preliminaries}
We first review some essential preliminaries.

\textbf{One-dimensional Wasserstein.} Let $\mu$ and $\nu$ be two one-dimensional measures belongs to $\mathcal{P}_p(\mathbb{R})$,  the Wasserstein distance has a closed form which is 
\begin{align*}
    \text{W}_p^p(\mu,\nu) &= \inf_{\pi \in \Pi(\mu,\nu)} \int_{\mathbb{R}\times \mathbb{R}}  |x-y|^pd\pi(x,y)\\& =
     \int_0^1 |F_{\mu}^{-1}(z) - F_{\nu}^{-1}(z)|^{p} dz 
\end{align*}
where $F_{ \mu}$ and $F_{ \nu}$  are  the cumulative
distribution function (CDF) of $ \mu$ and $ \nu$ respectively.

\textbf{Sliced Wasserstein. }  To utilize the closed-form, the sliced Wasserstein distance averages all possible projected Wasserstein distances i.e., the definition of sliced Wasserstein (SW) distance~\cite{bonneel2015sliced} between two probability measures $\mu \in \mathcal{P}_p(\mathbb{R}^d)$ and $\nu\in \mathcal{P}_p(\mathbb{R}^d)$ is:
\begin{align}
\label{eq:SW}
    \text{SW}_p^p(\mu,\nu)  =  \mathbb{E}_{ \theta \sim \mathcal{U}(\mathbb{S}^{d-1})} [\text{W}_p^p (\theta \sharp \mu,\theta \sharp \nu)],
\end{align}
where $\theta \sharp \mu$ is the push-forward measures of $\mu$ through the function $f:\mathbb{R}^{d} \to \mathbb{R}$ that is $f(x) = \theta^\top x$. However, the expectation in the definition of the SW distance is intractable to compute. Therefore, the Monte Carlo scheme is employed to approximate the value:
\begin{align*}
    \widehat{\text{SW}}_{p}^p(\mu,\nu;L) = \frac{1}{L} \sum_{l=1}^L\text{W}_p^p (\theta_l \sharp \mu,\theta_l \sharp \nu),
\end{align*}
where $\theta_{1},\ldots,\theta_{L} \overset{i.i.d}{\sim}\mathcal{U}(\mathbb{S}^{d-1})$ and are referred to as projecting directions. The pushfoward measures $\theta_1\sharp \mu,\ldots, \theta_L \sharp \mu$ are called projections of $\mu$ (similarly for $\nu$). The number of Monte Carlo samples $L$ is often referred to as the number of projections. When $\mu$ and $\nu$ are discrete measures that have at most $n$ supports, the time complexity and memory complexity of the SW are $\mathcal{O}(Ln\log n+Ldn)$ and $\mathcal{O}(Ld+Ln)$ respectively.

\textbf{Distributional sliced Wasserstein.} To select a better-slicing distribution,~\citet{nguyen2021distributional}  introduce the distributional sliced Wasserstein (DSW) distance which is defined between two probability measures $\mu \in \mathcal{P}_p(\mathbb{R}^d)$ and $\nu\in \mathcal{P}_p(\mathbb{R}^d)$ as:
\begin{align}
\label{eq:DSW}
    \text{DSW}_p^p(\mu,\nu)  = \max_{\psi \in \Psi } \mathbb{E}_{ \theta \sim \sigma_\psi(\theta))} [\text{W}_p^p (\theta \sharp \mu,\theta \sharp \nu)],
\end{align}
where $\sigma_\psi(\theta) \in \mathcal{P}(\mathbb{S}^{d-1})$, e.g., an implicit distribution~\cite{nguyen2021distributional}, von Mises-Fisher~\cite{jupp1979maximum} (vMF) distribution and Power Spherical (PS)~\cite{cao2018partial} with unknown location parameter $\sigma_\psi(\theta):=(\text{PS}) \text{vMF}(\theta|\epsilon,\kappa)$, $\psi =\epsilon$. After using $T\geq 1$ (projected) stochastic (sub)-gradient ascent iterations to obtain an estimation of the parameter $\hat{\psi}_T$, Monte Carlo samples $\theta_1,\ldots,\theta_L \overset{i.i.d}{\sim} \sigma_{\hat{\psi}_T}(\theta)$ are used to approximate the value of the DSW. The time complexity and space complexity of the DSW are $\mathcal{O}(LTn\log n +LTdn)$ and $\mathcal{O}(Ld+Ln)$  without counting the complexities of sampling from $\sigma_{\hat{\psi}_T}(\theta)$.

\textbf{Max sliced Wasserstein.} By letting the concentration parameter $\kappa \to \infty$, the vMF and PS distributions degenerate to the Dirac distribution, we obtain the max sliced Wasserstein distance~\cite{deshpande2019max}. The definition of max sliced Wasserstein (Max-SW) distance between two probability measures $\mu \in \mathcal{P}_p(\mathbb{R}^d)$ and $\nu\in \mathcal{P}_p(\mathbb{R}^d)$ is:
\begin{align}
\label{eq:MaxSW}
    \text{Max-SW}_p(\mu,\nu)  = \max_{\theta \in \mathbb{S}^{d-1}} \text{W}_p (\theta \sharp \mu,\theta \sharp \nu).
\end{align}
The Max-SW is computed by using $T\geq 1$ iterations of (projected) (sub)-gradient ascent to obtain the approximation of the ``max" projecting direction $\hat{\theta}_T$. The approximated value of the Max-SW is then set to $\text{W}_p (\hat{\theta}_T \sharp \mu,\hat{\theta}_T \sharp \nu)$.  It is worth noting that the optimization problem is non-convex~\cite{nietert2022statistical} which leads to the fact that we cannot obtain the global optimum $\theta^\star$. As a result, the approximation of Max-SW is not a metric even when we let the number of iterations $T \to \infty$.  The time complexity and space complexity of the Max-SW are $\mathcal{O}(Tn\log n +Tdn)$ and $\mathcal{O}(d+n)$. 

\textbf{Energy-based sliced Wasserstein.} In practice, optimization often costs more computation than sampling from a fixed slicing distribution. To avoid expensive optimization, the recent work~\cite{nguyen2023energy} proposes an optimization-free approach that is based on energy-based slicing distribution. The energy-based sliced Wasserstein (EBSW) distance between two probability measures $\mu \in \mathcal{P}_p(\mathbb{R}^d)$ and $\nu\in \mathcal{P}_p(\mathbb{R}^d)$ is defined as follow:
    \begin{align}
        &\text{EBSW}_p^p(\mu,\nu;f) = \mathbb{E}_{\theta \sim \sigma_{\mu,\nu}(\theta;f,p)}\left[ \text{W}_p^p (\theta\sharp \mu,\theta \sharp \nu)\right],
    \end{align}
where $f:[0,\infty) \to (0,\infty)$ is an increasing energy function e.g., $f(x)=e^x$, and $\sigma_{\mu,\nu}(\theta;f,p) \propto f(\text{W}^p_p(\theta \sharp \mu,\theta \sharp \nu))$. EBSW can be approximated by importance sampling with the uniform proposal distribution $\sigma_0 = \mathcal{U}(\mathbb{S}^{d-1})$. For $\theta_1,\ldots,\theta_L \overset{i.i.d}{\sim} \sigma_0(\theta)$, we have: $\widehat{\text{EBSW}}_p^p(\mu,\nu;f,L)= $
\begin{align}
\label{eq:emp_ISEBSW}
   \sum_{l=1}^L   \text{W}_p^p (\theta_l\sharp \mu,\theta_l \sharp \nu)\hat{w}_{\mu,\nu,\sigma_0,f ,p} (\theta_l) ,
\end{align}
for $w_{\mu,\nu,\sigma_0,f,p } (\theta) = \frac{f(\text{W}^p_p(\theta \sharp \mu,\theta \sharp \nu) )}{\sigma_0(\theta)}$ is the importance weighted function and $\hat{w}_{\mu,\nu,\sigma_0,f ,p} (\theta_l) = \frac{w_{\mu,\nu,\sigma_0,f,p } (\theta_l)}{\sum_{l'=1}^L w_{\mu,\nu,\sigma_0,f,p } (\theta_{l'})}$ is the normalized importance weights.  The time complexity and space complexity of the EBSW are also $\mathcal{O}(Ln\log n+Ldn)$ and $\mathcal{O}(Ld+Ln)$ in turn. However, the estimation is biased, i.e., $\mathbb{E}[\widehat{\text{EBSW}}_p^p(\mu,\nu;f,L)]\neq \text{EBSW}_p^p(\mu,\nu;f)$. In addition to importance sampling, Markov Chain Monte Carlo (MCMC) can be used to approximate EBSW, however, they are very computationally expensive~\cite{nguyen2023energy}.
%%%%%%%%%%%%%%%%%%%%%%%%%%%%%%%%%%%%%%%%%%%%%%%%%%%%%%%%%%%%%%%%%%%%%
\section{Random-Path Projection Sliced Wasserstein}
\label{sec:RPSW}
We first discuss the random-path projecting direction and the random-path slicing distribution in Section~\ref{subsec:rpd}. We then discuss the Random-Path Projection Sliced Wasserstein variants in Section~\ref{subsec:rpsw}.

\subsection{Random-Path Projecting Direction}
\label{subsec:rpd}
As discussed,  discriminative projecting directions, i.e., directions that can highlight the difference between two distributions after the projection,  are preferred. Although the energy-based slicing distribution in EBSW can do the job, it is not possible to sample directly from the energy-based slicing distribution. To address the issue, we propose a novel slicing distribution that is based on the new notion of projecting direction, referred to as  ``random-path".

\textbf{Random paths between two measures.} We now define the ``random path" between $\mu \in \mathcal{P}(\mathbb{R}^d)$ and $\nu \in \mathcal{P}(\mathbb{R}^d)$.
\begin{definition}[Random-path]
\label{def:rp}
For any $p \geq 1$, dimension $d \geq 1$,  $\mu \in \mathcal{P}(\mathbb{R}^d)$ and $\nu \in \mathcal{P}(\mathbb{R}^d)$, let $X \sim \mu$ and $Y\sim \nu$, the random-path (RP) is defined as:
\begin{align}
    Z_{X,Y} =X-Y.
\end{align}
\end{definition}

% From Definition~\ref{def:rpd}, the random-path $Z_{X,Y}$ belongs to the set of distribution over the unit-hypersphere $\mathcal{P}(\mathbb{S}^{d-1})$ for any $\mu$ and $\nu$ in $\mathcal{P}(\mathbb{R}^d)$.
The measure of $Z_{X,Y}$  can be written as $\sigma_{\mu-\nu}:=\mu * (-)\sharp \nu$ where $*$ denotes the convolution operator, and $(-)\sharp \nu$ denotes the pushforward measures of $\nu$ through the function $f(x)=-x$. The density of  $\sigma_{\mu-\nu}$ is intractable, however, it is simple to sample from it, i.e., sampling $X \sim \mu$, $Y \sim \nu$, set $Z_{X,Y} =X-Y$, then $Z_{X,Y}  \sim \sigma_{\mu-\nu}$. It is worth noting that $\sigma_{\mu-\nu}\neq \sigma_{\nu-\mu}$.

\textbf{Random-path projecting direction.} From the random path, we can create a random projecting direction by normalizing it i.e., $\frac{Z_{X,Y}}{\|Z_{X,Y}\|_2}$. However, the density of the RP, i.e., $\sigma_{\mu-\nu}$ does not always have full support on $\mathbb{R}^d$ e.g., in discrete cases ($\mu$ and $\nu$ are discrete). Therefore, $\frac{Z_{X,Y}}{\|Z_{X,Y}\|_2}$ does not fully support $\mathbb{S}^{d-1}$ which cannot guarantee theoretical property and harm practical performance. As a solution, we can add a random perturbation around the normalized random path. As long as the random perturbation has continuous density,  the marginal density of the final projecting direction is continuous. Now, we define the random-path projecting direction as follows:

\begin{definition}[Random-path projecting direction]
\label{def:rpd}
For any  $\kappa >0$, dimension $d \geq 1$,  $\mu \in \mathcal{P}(\mathbb{R}^d)$ and $\nu \in \mathcal{P}(\mathbb{R}^d)$, let $X \sim \mu$ and $Y\sim \nu$, the random-path projecting direction (RPD) is:
\begin{align}
    \theta|X,Y,\kappa \sim \sigma_\kappa(\cdot;P_{\mathbb{S}^{d-1}}(X-Y)),
\end{align}
where $P_{\mathbb{S}^{d-1}}(x) = \frac{x}{\|x\|_2}$, and $\sigma_\kappa$ is a location scale distribution on $\mathbb{S}^{d-1}$ such as vMF and PS.
\end{definition}

In Definition~\ref{def:rpd}, the reason for choosing the location-scale family is to guarantee the projecting direction concentrates around the normalized random path. We recall the density of vMF distribution $\text{vMF}(\theta;\epsilon,\kappa)\propto \exp(\kappa\epsilon^\top \theta)$~\cite{jupp1979maximum} and the PS distribution $\text{PS}(\theta;\epsilon,\kappa)  \propto (1+\epsilon^\top \theta)^\kappa$~\cite{de2020power}.

\begin{remark}
    \label{remark:rpd} When $\mu$ and $\nu$ have shared supports, $Z_{X,Y}$ has support at $0$ which makes $P_{\mathbb{S}^{d-1}}(x)$ undefined. For continuous cases, it is not the problem since the probability of $Z_{X,Y}=0$ is 0. For the discrete cases, we can solve the issue by simply adding a small constant to $Z_{X,Y}$ i.e., $Z_{X,Y}+c$, and treat it as the random path. We find that this is not the problem in practice since two interested distributions are often continuous or rarely have shared supports. 
\end{remark}

\textbf{Random-path slicing distribution.} From Definition~\ref{def:rpd}, we can obtain the slicing distribution of the RPD by marginalizing out $X,Y$. In particular, we have the random-path slicing distribution (RPSD) defined as follows:
\begin{definition}
\label{def:rpsd}
For any $0<\kappa<\infty$, dimension $d \geq 1$,  $\mu \in \mathcal{P}(\mathbb{R}^d)$ and $\nu \in \mathcal{P}(\mathbb{R}^d)$, let $X \sim \mu$ and $Y\sim \nu$, the random-path slicing distribution (RPSD) is:
\begin{align*}
    \sigma_{RP}(\theta;\mu,\nu,\sigma_\kappa) = \int\sigma_\kappa(\theta;P_{\mathbb{S}^{d-1}}(x-y)) d\mu(x) d\nu(y).
\end{align*}
\end{definition}
Although $\sigma_{RP}(\theta;\mu,\nu,\sigma_\kappa)$ is intractable, we can sample from it easily by sampling $X \sim \mu, Y \sim \nu$, then $\theta\sim \sigma_\kappa(P_{\mathbb{S}^{d-1}}(X-Y))$. The sampling process can also be parallelized for multiple samples $\theta_1,\ldots,\theta_L$ for  $L>1$.

\begin{remark}
    \label{remark:rpsd} We can rewrite $\sigma_{RP}(\theta;\mu,\nu,\sigma_\kappa) = \int\sigma_\kappa(\theta;P_{\mathbb{S}^{d-1}}(x-y)) d \pi(x,y)$ where $\pi = \mu \otimes \nu$ is the independent coupling between $\mu$ and $\nu$ (efficiently for computation). As a natural extension, we can use other coupling between $\mu$ and $\nu$. However, a more complicated coupling could cost more computation for doing sampling while the benefit of using such coupling is not trivial. 
\end{remark}

\subsection{Random-Path Projection Sliced Wasserstein}
\label{subsec:rpsw}
We now discuss the two sliced Wasserstein variants that are based on the random-path slicing distribution.

\textit{\textbf{Definitions.}}
We first define the random-path projection sliced Wasserstein (RPSW) distance.
\begin{definition}
    \label{def:rpsw}
    For any $p \geq 1$, $d \geq 1$, $0<\kappa<\infty$, two probability measures $\mu \in \mathcal{P}_p(\mathbb{R}^d)$ and $\nu \in \mathcal{P}_p(\mathbb{R}^d)$,  the random-path projection sliced Wasserstein (RPSW) distance  is defined as follows:
    \begin{align*}
        &\text{RPSW}^p_p(\mu,\nu;\sigma_\kappa) =\mathbb{E}_{\theta \sim \sigma_{RP}(\theta;\mu,\nu,\sigma_\kappa)}[W_p^p(\theta \sharp \mu, \theta \sharp \nu)], 
    \end{align*}
    where $\sigma_{RP}$ is defined as in Definition~\ref{def:rpsd}.
\end{definition}
\label{submission}

\begin{remark}
    \label{remark:RPSW} The equivalent definition of RPSW to Definition~\ref{def:rpsw} is: 
    $
        \text{RPSW}^p_p(\mu,\nu;\sigma_\kappa) =\mathbb{E}_{X \sim \mu,Y\sim \nu} \mathbb{E}_{\theta \sim\sigma_\kappa(\theta;P_{\mathbb{S}^{d-1}}(X-Y)) }[W_p^p(\theta \sharp \mu, \theta \sharp \nu)].
    $
\end{remark}

Motivated by the importance of sampling estimation of EBSW in Equation~\eqref{eq:emp_ISEBSW}, we can further adjust the weight of RPDs i.e., $\theta_1,\ldots,\theta_L \sim \sigma_{RP}(\theta;\mu,\nu,\sigma_\kappa)$ by their corresponding projected distance $W_p^p(\theta_1 \sharp \mu, \theta_1 \sharp \nu)),\ldots, W_p^p(\theta_L \sharp \mu, \theta_L \sharp \nu))$. We now define the importance weighted random-path projection sliced Wasserstein (IWRPSW) distance.

\begin{definition}
    \label{def:iwrpsw}
    For any $p \geq 1$, $d \geq 1$, $L \geq 1$, $0<\kappa<\infty$, an increasing function $f:[0,\infty) \to (0,\infty)$, two probability measures $\mu \in \mathcal{P}_p(\mathbb{R}^d)$ and $\nu \in \mathcal{P}_p(\mathbb{R}^d)$,  the importance weighted random-path projection sliced Wasserstein (IWRPSW) is defined as:
    \begin{align*}
        &\text{IWRPSW}^p_p(\mu,\nu;\sigma_\kappa,L,f) =\mathbb{E}_{\theta_1,\ldots, \theta_L\sim \sigma_{RP}(\theta;\mu,\nu,\sigma_\kappa)} \nonumber\\
        &\quad \left[\sum_{l=1}^L W_p^p(\theta_l \sharp \mu, \theta_l \sharp \nu) \frac{f(W_p^p(\theta_l \sharp \mu, \theta_l \sharp \nu))}{\sum_{j=1}^L f(W_p^p(\theta_j \sharp \mu, \theta_j \sharp \nu))} \right], 
    \end{align*}
    where $\sigma_{RP}$ is defined as in Definition~\ref{def:rpsd}.
\end{definition}
Compared to only one random projecting direction in RPSW, IWRPSW utilizes $L$ random projecting directions in the \textit{population} form. 

\textit{\textbf{Topological Properties.}} We first investigate the metricity of  RPSW and IWRPSW.
\begin{theorem} For any $p \geq 1$, $L \geq 1$, $f:[0,\infty) \to (0,\infty)$ and $0<\kappa<\infty$, the random-path projection sliced Wasserstein $\text{RPSW}_{p}(\cdot,\cdot;\sigma_\kappa)$ and the importance weighted random-path projection sliced Wasserstein $\text{IWRPSW}_{p}(\cdot,\cdot;\sigma_\kappa,L)$ are semi-metrics in the probability space on $\mathbb{R}^{d}$, namely RPSW and IWRPSW satisfy non-negativity, symmetry, and identity of indiscernible.  The RPSW  satisfies the ``quasi"-triangle inequality i.e., \begin{align*}
    \text{RPSW}_{p}(\mu_1,\mu_2;\sigma_\kappa) &\leq \text{RPSW}_{p}(\mu_1,\mu_3;\sigma_\kappa,\mu_1,\mu_2)  \\&+ \text{RPSW}_{p}(\mu_3,\mu_2;\sigma_\kappa,\mu_1,\mu_2)
\end{align*} where we have $\text{RPSW}_{p}^p(\mu_1,\mu_3;\sigma_\kappa,\mu_1,\mu_2) = \mathbb{E}_{\theta \sim \sigma_{RP}(\theta;\mu_1,\mu_2,\sigma_\kappa)}[W_p^p(\theta \sharp \mu_1, \theta \sharp \mu_3)]$ and a  similar definition of $\text{RPSW}_{p}^p(\mu_3,\mu_2;\sigma_\kappa,\mu_1,\mu_2)$.
\label{theo:metricity}
\end{theorem}
The proof of Theorem~\ref{theo:metricity} is given in Appendix~\ref{subsec:proof:theo:metricity}. It is worth noting that the triangle inequality for RPSW and IWRPSW is challenging to prove due to the dependency of the two input measures with the slicing distribution. Therefore, it is unknown if they satisfy the triangle inequality.
\begin{remark}
    \label{remark:symmetry} Although the $\sigma_{RP}(\theta;\mu,\nu,\sigma_\kappa)$ is not symmetric with respective to $\mu$ and $\nu$, RPSW and IWRPSW are still symmetric since $W_p^p(\theta \sharp \mu, \theta \sharp \nu)$ is symmetric with respect to $\theta$ i.e., $W_p^p(\theta \sharp \mu, \theta \sharp \nu)=W_p^p(-\theta \sharp \mu, -\theta \sharp \nu)$ and we have $\sigma_\kappa(\theta;P_{\mathbb{S}^{d-1}}(x-y)) = \sigma_\kappa(-\theta;P_{\mathbb{S}^{d-1}}(y-x))$. We refer to the proof of Theorem~\ref{theo:metricity} for a more detail.
\end{remark}

We now discuss the connection between RPSW and IWRPSW, and their connection to other SW variants and Wasserstein distance.
\begin{proposition}
    \label{prop:connection}
    For any $p \geq 1$, $L \geq 1$, increasing function $f:[0,\infty) \to (0,\infty)$, and $0<\kappa<\infty$, we have the following relationship: \\ $\text{(i) } \text{RPSW}_{p}(\mu, \nu;\sigma_\kappa) \leq \text{IWRPSW}_{p}(\mu, \nu;\sigma_\kappa,L,f)  \leq \text{Max-SW}_p(\mu,\nu) \leq W_{p}(\mu, \nu),$\\
    $\text{(ii) } \lim_{\kappa \to 0} \text{RPSW}_{p}(\mu, \nu;\sigma_\kappa) \to \text{SW}_p(\mu,\nu),$   \\
    $\text{(iii) } \lim_{L \to \infty} \text{IWRPSW}_{p}(\mu, \nu;\sigma_\kappa,L,f) \to \text{EBSW}_p(\mu,\nu;f),$  \\
(iv)  $\text{RPSW}_{p}(\mu, \nu;\sigma_\kappa) \geq \text{RPSW}_{1}(\mu, \nu;\sigma_\kappa)$.
    
\end{proposition}
The proof of Proposition~\ref{prop:connection} is given in Appendix~\ref{subsec:proof:prop:connection}. 

% \begin{theorem}
%     \label{theo:weak}
% For any $p \geq 1$, the convergence of probability measures under the random-path sliced Wasserstein distance $\text{RPSW}_{p}(\cdot,\cdot)$ implies weak convergence of probability measures and vice versa.
% \end{theorem}

% The proof of Theorem~\ref{theo:weak} is given in Appendix~\ref{subsec:proof:theo:weak}.
\textit{\textbf{Statistical Properties.}} We now discuss the one-sided sample complexity of RPSW and IWRPSW.

\begin{proposition}
    \label{prop:sample_complexity}
    Let $X_{1}, X_{2}, \ldots, X_{n}$  be i.i.d. samples from the probability measures $\mu$ being supported on compact set of $\mathbb{R}^{d}$.  We denote the empirical measures $\mu_{n} = \frac{1}{n} \sum_{i = 1}^{n} \delta_{X_{i}}$. Then, for any $p > 1$, $L\geq 1$ and  $0<\kappa<\infty$, there exists a universal constant $C > 0$ such that
\begin{align*}
    \mathbb{E} [\text{RPSW}_{p}(\mu_{n},\mu;\sigma_\kappa)] &\leq \mathbb{E} [\text{IWRPSW}_{p} (\mu_{n},\mu;\sigma_\kappa,L,f)] \\&\leq C \sqrt{\frac{(d+1)\log (n+1)}{n}},
\end{align*}
where the outer expectation is taken with respect to $X_{1}, X_{2}, \ldots, X_{n}$.
\end{proposition}

The proof of Proposition~\ref{prop:sample_complexity} is given in Appendix~\ref{subsec:proof:prop:sample_complexity}. From the proposition, we can see that the approximation rate of using an empirical probability measure to a population measure of the proposed RPSW and IWRPSW is at the order of $n^{-1/2}$. Therefore, it is possible to say that the proposed RPSW and IWRPSW do not suffer from the curse of dimensionality as other SW variants. Due to the missing proof of triangle inequality of RPSW and IWRPSW, it is non-trivial to extend from the one-sided sample complexity to the two-sided sample complexity as in the conventional SW~\cite{nadjahi2019asymptotic}.

\textit{\textbf{Computational Properties.}} We now present how to compute RPSW and IWRPSW in practice.

\textbf{Monte Carlo estimation.} The expectations in RPSW is intractable (Definition~\ref{def:rpsw}), hence, Monte Carlo samples are used to estimate RPSW. In particular, we sample $\theta_1,\ldots,\theta_L \sim \sigma_{RP}(\theta;\mu,\nu,\kappa)$ as described in Section~\ref{subsec:rpd}, then we form the Monte Carlo estimate of RPSW:
\begin{align}
    \label{eq:MC_RPSW}
    \widehat{\text{RPSW}}_p^p(\mu,\nu;\sigma_\kappa,L)  =\frac{1}{L} \sum_{l=1}^L W_p^p(\theta_l \sharp \mu,\theta_l \sharp \nu).
\end{align}
We refer to the reader to Algorithm~\ref{alg:RPSW}-~\ref{alg:IWRPSW} in Appendix~\ref{sec:algorithm} for the detailed computation of RPSW and IWRPSW. We then discuss the approximation error of the estimator.
\begin{proposition}
    \label{proposition:MCerror}
    For any $p\geq 1$, $0<\kappa<\infty$ , dimension $d \geq 1$, and $\mu,\nu \in \mathcal{P}_p(\mathbb{R}^d)$, we have:
    \begin{align*}
        &\mathbb{E} | \widehat{\text{RPSW}}_{p}^p(\mu,\nu;\sigma_\kappa) - \text{RPSW}_{p}^p (\mu,\nu;\sigma_\kappa)|  \\&\quad \leq \frac{1}{\sqrt{L}} Var_{\theta \sim \sigma_{RP}(\theta;\mu,\nu,\sigma_\kappa)}\left[ W_p^p \left(\theta \sharp \mu, \theta \sharp \nu \right)\right]^{\frac{1}{2}}.
    \end{align*}
\end{proposition}
The proof of Proposition~\ref{proposition:MCerror} is given in Appendix~\ref{proposition:MCerror}. From the proposition, we observe that RPSW has the same error rate as SW in terms of $L$ i.e., $L^{-1/2}$.

Similarly for IWRPSW, let $H\geq 1$,
we sample $\theta_{11},\ldots,\theta_{HL} \sim \sigma_{RP}(\theta;\mu,\nu,\kappa)$, then we form the Monte Carlo estimate of IWRPSW as follows:
\begin{align}
    \label{eq:MC_IWRPSW}
    &\widehat{\text{IWRPSW}}_p^p(\mu,\nu;\sigma_\kappa,L,H) =\frac{1}{H}\sum_{h=1}^H\nonumber\\ &\left[\sum_{l=1}^L W_p^p(\theta_{hl} \sharp \mu, \theta_{hl} \sharp \nu) \frac{f(W_p^p(\theta_{hl} \sharp \mu, \theta_{hl} \sharp \nu))}{\sum_{j=1}^L f(W_p^p(\theta_{hj} \sharp \mu, \theta_{hj} \sharp \nu))} \right].
\end{align}
Since IWRPSW has $L$ random projecting directions, we approximate the expectation by $H$ sets of Monte Carlo samples. To simplify and make $L$ as the only parameter for the number of projections, we set $H=1$ in this paper. In contrast to RPSW, the error rate of IWRPSW with respect to $L$ is non-trivial due to the importance weights.

\textbf{Unbiasedness.} Since $\text{RPSW}^p_p(\mu,\nu;\sigma_\kappa) $ and $\text{IWRPSW}^p_p(\mu,\nu;\sigma_\kappa,L,f) $ can be approximated directly by Monte Carlo samples, their corresponding estimators $\widehat{\text{RPSW}}_{p}^p(\mu,\nu;\sigma_\kappa)$ and $\widehat{\text{IWRPSW}}_p^p(\mu,\nu;\sigma_\kappa,L,H) $ are unbiased estimates. 

\textbf{Computational complexities.} When $\mu$
 and $\nu$ are discrete measures that have a most $n$ supports, sampling from them only cost $\mathcal{O}(n)$ in terms of time and space. Hence, sampling $L$ random paths between $\mu$ and $\nu$ cost $\mathcal{O}(Ldn)$ in time and memory. After that, sampling from the vMF distribution and the PS distribution costs $\mathcal{O}(Ld)$ (Algorithm 1 in~\cite{de2020power}) in time and memory. Adding the complexities for computing one-dimensional Wasserstein distance, the time complexity and space complexity for RPSW is $\mathcal{O}(Ln\log n+Ldn)$ and $\mathcal{O}(Ld+Ln)$ respectively. For IWRPSW, the complexities are multiplied by $H$, however, they can be kept the same if setting $H=1$. 

\textbf{Gradient estimation.} In applications where RPSW and IWRPSW are used as a risk to estimate some parameters of interest i.e., $\mu_\phi$ with $\phi \in \Phi$, we might want to estimate the gradient of RPSW and IWRPSW with respect to $\phi$. For RPSW, we have:
\begin{align*}
    &\nabla_\phi \text{RPSW}^p_p(\mu_\phi,\nu;\sigma_\kappa) \\
    &=\nabla_\phi \mathbb{E}_{X \sim \mu_\phi ,Y\sim \nu} \mathbb{E}_{\theta \sim\sigma_\kappa(\theta;P_{\mathbb{S}^{d-1}}(X-Y)) }[W_p^p(\theta \sharp \mu_\phi, \theta \sharp \nu)].
\end{align*}
Since vMF and PS are reparameterizable~\cite{de2020power}, we can reparameterize $\sigma_{RP}(\theta;\mu,\nu,\sigma_\kappa)$ if $\mu_\phi$ is reparameterizable (e.g., $\mu_\phi:=f_\phi \sharp \varepsilon$ for $\varepsilon$ is a fixed distribution, or discussed other distributions in~\cite{kingma2013auto}). With parameterized sampling, we can sample $\theta_{1,\phi},\ldots,\theta_{L,\phi} \sim \sigma_{RP}(\theta;\mu_\phi,\nu,\sigma_\kappa)$, then form an unbiased gradient estimator as follow:
\begin{align*}
    \nabla_\phi \text{RPSW}^p_p(\mu_\phi,\nu;\sigma_\kappa)  \approx \frac{1}{L} \sum_{l=1}^L \nabla_\phi W_p^p(\theta_{l,\phi}\sharp \mu_\phi,\theta_{l,\phi}\sharp \nu).
\end{align*}
For IWRPSW, we sample $\theta_{11,\phi},\ldots,\theta_{HL,\phi} \sim \sigma_{RP}(\theta;\mu_\phi,\nu,\kappa)$, then form the following estimator:
\begin{align*}
    &\nabla_\phi\text{IWRPSW}_p^p(\mu_\phi,\nu;\sigma_\kappa,L,H) =\frac{1}{H}\sum_{h=1}^H\left[\nabla_\phi\sum_{l=1}^L  \right.\nonumber\\ &
    \left.W_p^p(\theta_{hl,\phi} \sharp \mu_\phi, \theta_{hl,\phi} \sharp \nu) \frac{f(W_p^p(\theta_{hl,\phi} \sharp \mu_\phi, \theta_{hl,\phi} \sharp \nu))}{\sum_{j=1}^L f(W_p^p(\theta_{hj,\phi} \sharp \mu_\phi, \theta_{hj,\phi} \sharp \nu))} \right].
\end{align*}

\begin{remark}
\label{remark:gradient}
    We found in later experiments that removing the dependent of $\phi$ in the slicing distribution i.e., using a dependent copy $\sigma_{RP}(\theta;\mu_{\phi'},\nu,\kappa)$ with $\phi'=\phi$ in value, leads to a more stable estimator in practice. It is worth noting that using a copy of both $\mu$ and $\nu$ (i.e., $\mu'$ and $\nu'$) in RPSW can even lead to a variant that satisfies the triangle inequality since the slicing distribution becomes a fixed distribution in this situation. We refer to this gradient estimator as the simplified gradient estimator.
\end{remark} 
\section{Experiments}
\label{sec:experiments}
In this section, we compare the proposed RPSW and IWRPSW to the existing SW variants such as SW, Max-SW, DSW, and EBSW in gradient flow in Section~\ref{subsec:gf} and training denoising diffusion models in~\ref{subsec:ddpm}. 

\subsection{Toy Gradient Flows}
\label{subsec:gf}
 \begin{figure*}[t]
\begin{center}
  \begin{tabular}{cc}
  \widgraph{0.45\textwidth}{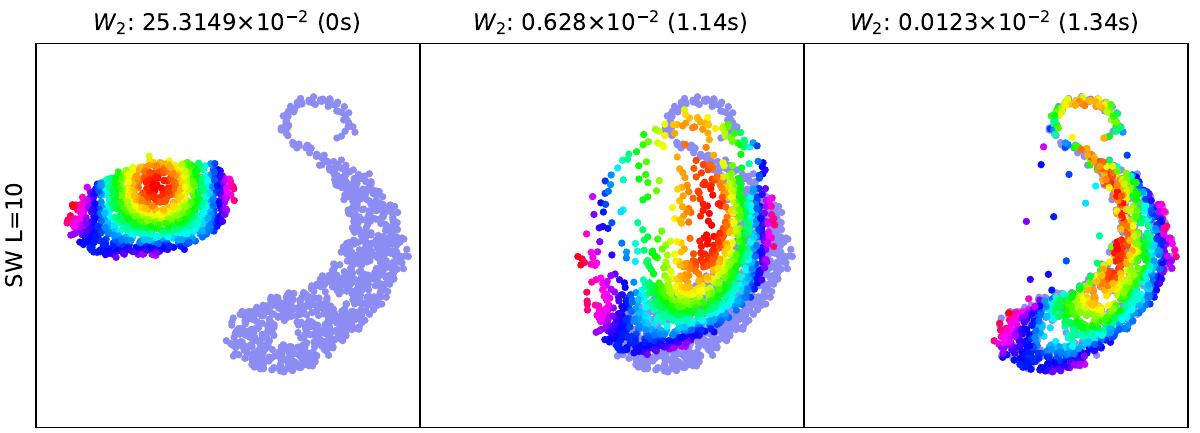}  & \widgraph{0.45\textwidth}{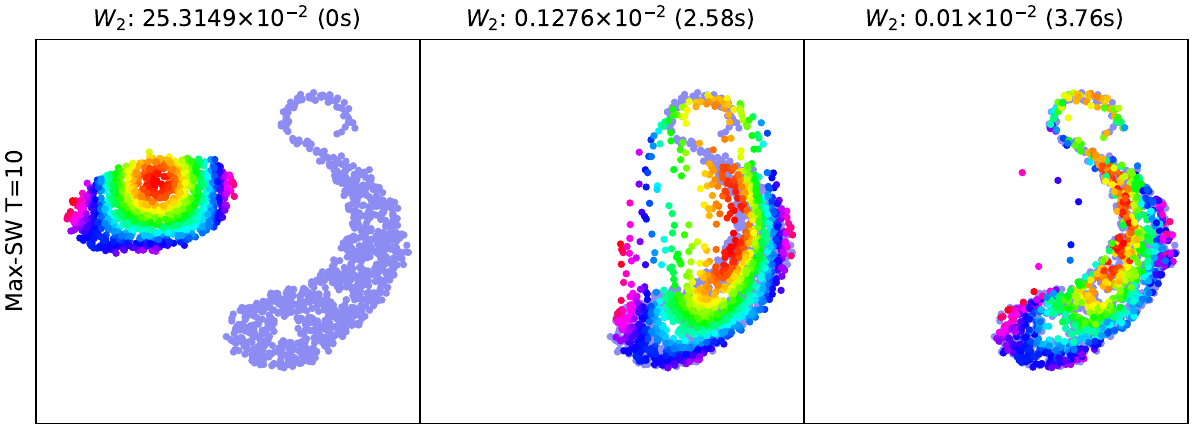} \\
  \widgraph{0.45\textwidth}{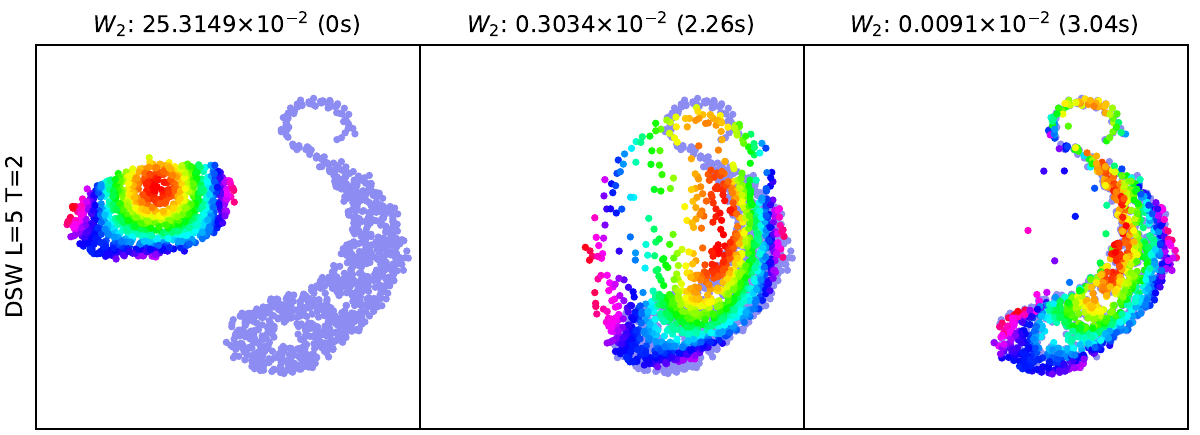}  & \widgraph{0.45\textwidth}{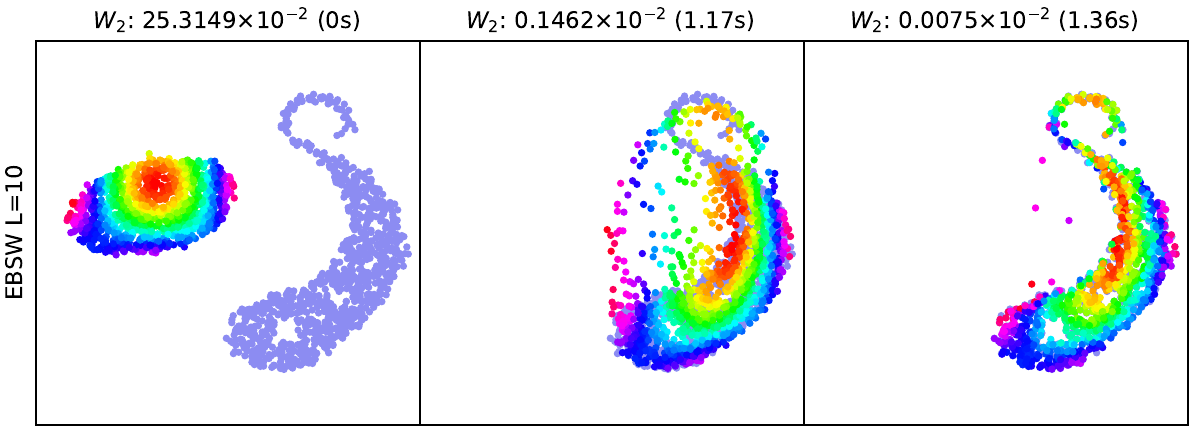}  \\
  \widgraph{0.45\textwidth}{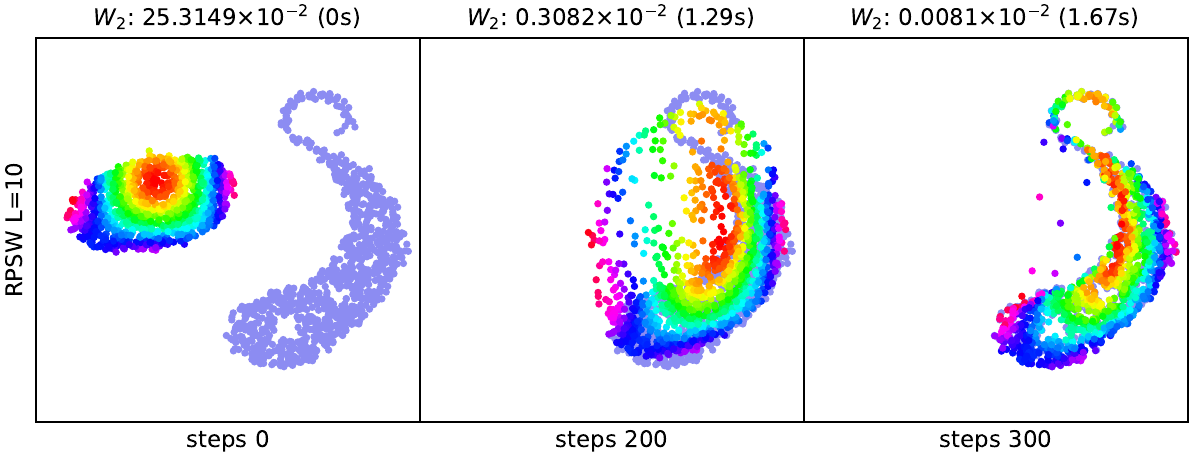}  & \widgraph{0.45\textwidth}{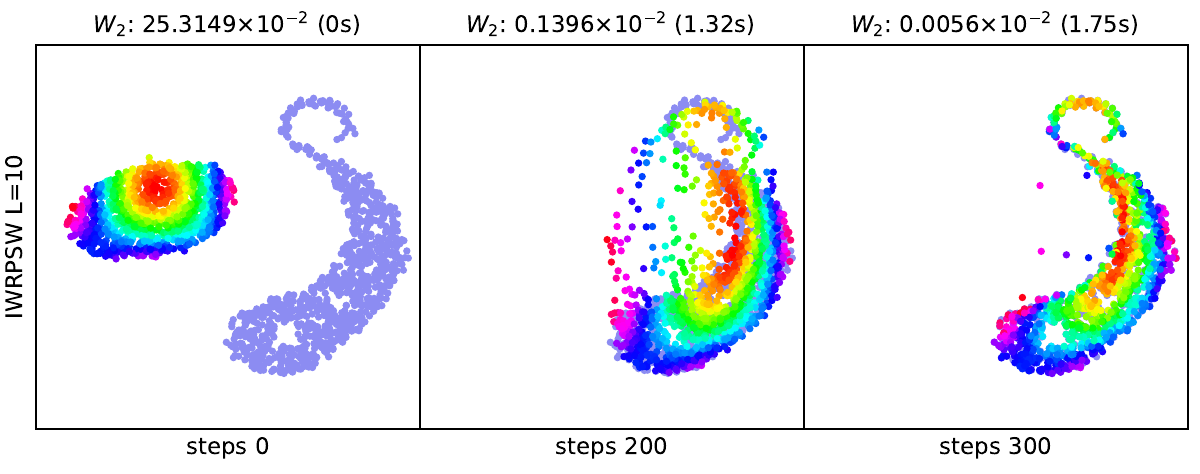} \\
  \end{tabular}
  \end{center}
  \vskip -0.2in
  \caption{
  {Results for gradient flows that are from the empirical distribution over the color points to the empirical distribution over S-shape points produced by different SW variants. 
  The corresponding Wasserstein-2 distance between the empirical distribution at the current step and the S-shape distribution and the computational time (in second) to reach the step is reported at the top of the figure.
}
} 
  \label{fig:gf_main}
  \vskip -0.15in
\end{figure*}
The gradient flow models a dynamic distribution $\mu(t)$ flowing with time $t$ along the gradient flow of a loss functional $\mu(t) \to f(\mu(t))$. In this experiment, we set $f(\mu(t))=\mathcal{D}(\mu(t),\nu)$ with $\nu$ as the target distribution and  $\mathcal{D}$ is a given SW variant between probability measures. The flow will drive the source distribution  towards $\nu$~\cite{feydy2019interpolating,santambrogio2015optimal}. In this setup, we consider the discrete setting i.e., $\nu =  \frac{1}{n}\sum_{i=1}^n \delta_{Y_i}$  and $\mu(t) = \frac{1}{n} \sum_{i=1}^n \delta_{X_i(t)}$.  

\textbf{Setting.} To solve the flow,  we integrate the ordinary differential equation  $\dot{X}(t)=- n \nabla_{X(t)} \left[\mathcal{D}\left(\frac{1}{n } \sum_{i=1}^n \delta_{X_i(t)}, \nu \right)\right] $ with the Euler scheme with $300$ timesteps and the step size is $10^{-4}$. We set the total number of projections for SW variants to 10. For DSW, RPSW, and IWRPSW, we set the concentration parameter $\kappa$ of the PS distribution as a dynamic quantity i.e., $\kappa(t)=(\kappa_0-1) \left(\frac{N-t-1}{N-1}\right)^{10} +1$ with $N=300$ and $\kappa_0 \in \{100,50\}$. The reason for the dynamic choice of $\kappa$ is to stabilize the convergence of the flow since when two distributions are relatively closed, the uniform distribution is nearly optimal.

\textbf{Results.} We show both the qualitative visualization and quantitative comparison (in Wasserstein-2 distance~\cite{flamary2021pot}) in Figure~\ref{fig:gf_main}. From the figure, we observe that RPSW leads to a faster and better convergence than SW and IWRPSW leads to a faster and better convergence than EBSW. We also observe that optimization-based variants such as Max-SW and DSW are slower than sampling-based variants like SW, EBSW, RPSW, and IWRPSW. Moreover, the convergence of Max-SW and DSW is not as good as EBSW, RPSW, and IWRPSW. Overall, the experiment has shown the benefit of the random-path slicing distribution which improves the performance while having a fast computation. In the experiment, we use the simplified gradient estimator for RPSW and IWRPSW. We refer the reader to Figure~\ref{fig:gf_grad} in Appendix~\ref{sec:add_exps} for the result of the original gradient estimator. Despite the fact that such estimators can lead to faster convergence, they seem to be worse in remaining the original topology of the source distribution.  

\begin{table*}[!t]
    \centering
    \caption{\footnotesize{Wasserstein-2 distance and computational times across timesteps in gradient flow on MNIST.}}
    \scalebox{0.75}{
    \begin{tabular}{l|cc|cc|cc|cc|cc|}
    \toprule
     \multirow{2}{*}{Distance}&\multicolumn{2}{c|}{Step 100}&\multicolumn{2}{c|}{Step 500}&\multicolumn{2}{c|}{Step  1000}&\multicolumn{2}{c|}{Step  3000 }&\multicolumn{2}{c|}{Step  5000 }\\
     \cmidrule{2-11}
     & $\text{W}_2$($\downarrow$) &$\text{Time (s)}$($\downarrow$) & $\text{W}_2$($\downarrow$) &$\text{Time (s)}$($\downarrow$)& $\text{W}_2$($\downarrow$) &$\text{Time (s)}$($\downarrow$) & $\text{W}_2$($\downarrow$) &$\text{Time (s)}$($\downarrow$) &$\text{W}_2$($\downarrow$) &$\text{Time (s)}$($\downarrow$)\\
     \midrule
    SW& 		83.86&\textbf{7.11}&	67.74&\textbf{16.76}	&49.34&\textbf{28.09}&	29.14&\textbf{67.65}&	25.17&\textbf{107.83}\\
    Max-SW&	\textbf{34.74}&128.88&	33.79 &753.47&	33.42&1435.14&	33.00&38145.11	&32.37&6436.89\\
    DSW & 	50.26&136.14&	48.22 &462.57	&40.71 &764.98	&29.96&202.79	&27.49&291.20\\
    EBSW&80.29&7.25	&67.13&16.99	&48.69&28.51&	29.03&68.53	&25.11&108.65 \\
    RPSW &46.54&7.58&	28.97&18.13	&27.57&30.61	&23.83&74.71&	21.86&118.83\\
    IWRPSW & 44.28&7.59&	\textbf{28.75}&18.20	&2\textbf{7.49}&30.66&	\textbf{23.82}&77.71&	\textbf{21.81}&120.05\\
   % SIR-EBSW-e (ours) & $2.00\pm 0.03$ & $9.72\pm 0.04$ \\
   %  IMH-EBSW-e (ours) &$\mathbf{1.99\pm 0.05}$ & $9.72\pm 0.10$ \\
    \bottomrule
    \end{tabular}
}
    \label{tab:MNIST}
    % \vskip -0.1in
\end{table*}

 \begin{figure*}[!t]
\begin{center}
  \begin{tabular}{cccc}
  \widgraph{0.2\textwidth}{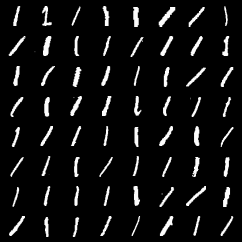}  & \widgraph{0.2\textwidth}{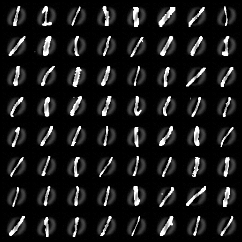} & \widgraph{0.2\textwidth}{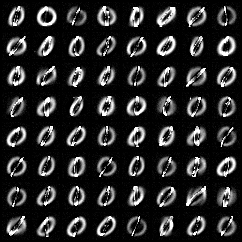}  & \widgraph{0.19\textwidth}{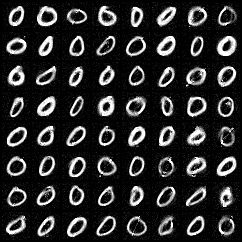}\\
SW  (Step 0) & SW (Step 100) & SW (Step 1000) & SW (Step 5000)\\

% \widgraph{0.19\textwidth}{images/RPSW_L1000_0.png}  & \widgraph{0.19\textwidth}{images/RPSW_L1000_99.png} & \widgraph{0.19\textwidth}{images/RPSW_L1000_999.png}  & \widgraph{0.19\textwidth}{images/RPSW_L1000_4999.png}\\
% RPSW (Step 0) & RPSW (Step 100) & RPSW (Step 1000) & RPSW (Step 5000)\\

\widgraph{0.19\textwidth}{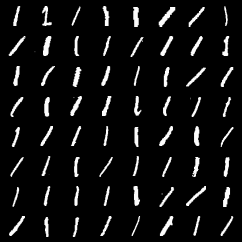}  & \widgraph{0.19\textwidth}{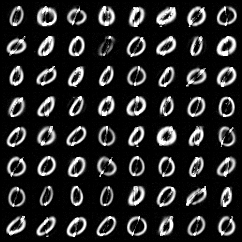} & \widgraph{0.19\textwidth}{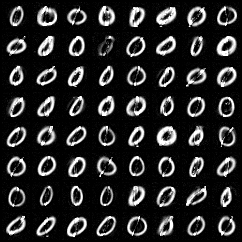}  & \widgraph{0.19\textwidth}{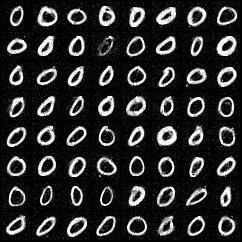}\\
IWRPSW (Step 0) & IWRPSW  (Step 100) & IWRPSW  (Step 1000) & IWRPSW (Step 5000)\\

  \end{tabular}
  \end{center}
  \vskip -0.1in
  \caption{Gradient flows from MNIST digit 1 to MNIST digit 0.
  {.
}
} 
  \label{fig:gf_mnist_1000_main}
  \vskip -0.2in
\end{figure*}

\subsection{Gradient Flows on Images}
\label{subsec:gf_MNIST}
\textbf{Setting.} Utilizing MNIST dataset~\cite{lecun1998gradient}, we select images of digit 1 to construct the source distribution and images of digit 0 to construct the target distribution. Following the previous section, we use the Euler discretization scheme with step size 1 and $N=5000$ iterations. We compare SW variants i..e, SW, EBSW, RPSW, and IWRPSW with the number of projections $L=1000$, Max-SW with $T=1000$, DSW with $L=100, T=10$. For DSW, RPSW, and IWRPSW, we set $\kappa(t)=(\kappa_0-1) \left(\frac{N-t-0.001}{N-1}\right)^{10} +0.001$ with $N=5000$ and $\kappa_0 \in {1000}$ For evaluation, we use the Wasserstein-2 distance.

\textbf{Results.} We report the Wasserstein-2 distance across timesteps and the computational time in Table~\ref{tab:MNIST}. Moreover, we visualize the flows in Figure~\ref{fig:gf_mnist_1000_main} and Figure~\ref{fig:gf_mnist_1000} in Appendix~\ref{sec:add_exps}.  We see that RPSW and IWRPSW help to reduce the Wasserstein-2 very fast compared to  other optimization-free variants i.e., SW and EBSW. Moreover,  RPSW and IWRPSW have significantly lower computation than DSW and Max-SW.  It is worth noting that the quality of the gradient flow can be improved using convolution slicing operator as in~\cite{nguyen2022revisiting,du2023nonparametric}.

\subsection{Denoising Diffusion Models}
\label{subsec:ddpm}

Denoising diffusion model~\cite{ho2020denoising,sohl2015deep} defines a forward process that gradually adds noise to the data  $x_0\sim q(x_0)$. The process has $T>0$ steps and the noise is  Gaussian. 
\begin{align*}
    q(x_{1:T}|x_0) = \prod_{t\geq 1}q(x_t|x_{t-1}),
\end{align*}
where $q(x_t|x_{t-1}) = \mathcal{N}(x_t;\sqrt{1-\beta_t}x_{t-1},\beta_t I)$ with the pre-defined variance schedule $\beta_t$. Training denoising diffusion model is to estimate some parameters $\phi$ of a reverse denoising process defined by:
\begin{align*}
    p_\phi(x_{0:T}) =p(x_T) \prod_{t\geq 1} p_\phi(x_{t-1}|x_t),
\end{align*}
where $p_\phi(x_{t-1}|x_t) =\mathcal{N}(x_{t-1};\mu_\phi(x_t,t),\sigma_t^2I)$. The conventional training uses the maximum likelihood approach by  maximizing
the evidence lower bound $\mathcal{L} \leq p_\phi(x_0) = \int p_\phi(x_{0:T}) dx_{1:T}$, which can be rewritten as:
\begin{align*}
    \mathcal{L} = - \sum_{t \geq 1} \mathbb{E}_{q(x_t)} [KL(q(x_{t-1}|x_t)|| p_\phi(x_{t-1}|x_t))]+C,
\end{align*}
where $KL$ denotes the Kullback-Leibler divergence, and $C$ is a known constant.

 \begin{figure*}[!t]
\begin{center}
  \begin{tabular}{ccc}
  \widgraph{0.3\textwidth}{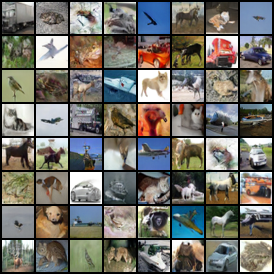}  &\widgraph{0.3\textwidth}{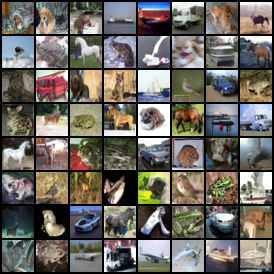}&\widgraph{0.3\textwidth}{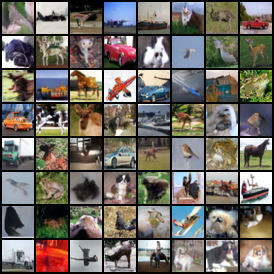} 
  \\
  DDGAN&RPSW-DD&IWRPSW-DD
  \end{tabular}
  \end{center}
  \vskip -0.15in
  \caption{
  {Random generated images on CIFAR10 from DDGAN, RPSW-DD, and IWRPSW-DD.
}
} 
  \label{fig:cifar10}
  \vskip -0.2in
\end{figure*}
\textbf{Implicit denoising model.} To reduce the number of steps $T$ for faster generation,  denoising diffusion GANs~\cite{xiao2021tackling} proposes to use the implicit denoising model $p_\phi(x_{t-1}|x_t) = \int p(\epsilon) q(x_{t-1}|x_t,x_0 = G_\phi(x_t,\epsilon,t))$  where $\epsilon \sim \mathcal{N}(0,I)$. After that,~\citet{xiao2021tackling} use adversarial  training to estimate the parameters i.e,
\begin{align*}
    \min_{\phi} \sum_{t \geq 1} \mathbb{E}_{q(x_t)} [D_{adv}(q(x_{t-1}|x_t)|| p_\phi(x_{t-1}|x_t))],
\end{align*}
where $D_{adv}$ is the GAN objective or the Jensen Shannon divergence~\cite{goodfellow2014generative}.

\textbf{Augmented Generalized Mini-batch Energy distances.} In this paper, we replace $D_{adv}$ by the generalized mini-batch Energy distance~\cite{salimans2018improving}. In particular, for two measures $\mu$ and $\nu$, and mini-batch size $m\geq 1$, we have the generalized Mini-batch Energy distance with SW kernel:
\begin{align*}
    &GME_{\mathcal{D}}^2(\mu,\nu) = 2\mathbb{E}[\mathcal{D}(P_{X},P_{Y}) ] \\&\quad- \mathbb{E}[\mathcal{D}(P_{X},P_{X^{'}}) -\mathbb{E}[\mathcal{D}(P_{Y},P_{Y^{'}})],
\end{align*}
where $X,X' \overset{i.i.d}{\sim} \mu^{\otimes m}$ and $X,X' \overset{i.i.d}{\sim} \nu^{\otimes m}$, $P_X = \frac{1}{m}\sum_{i=1}^m \delta_{x_i}$ with $X=(x_1,\ldots,x_m)$, and $\mathcal{D}$ is a SW variant. In the denoising diffusion case, $\mu$ is $q(x_{t-1}|x_t)$ and $\nu$ is $p_\phi(x_{t-1}|x_t))$ for $t=1,\ldots,T$. Although the generalized Mini-batch Energy distance can guarantee statistical convergence, it is known to be not good in practical training due to the weak training signal (lack of non-linearity which is essential for images)~\cite{salimans2018improving}. To address the issue, we propose the augmented generalized Mini-batch Energy distance:
$
    AGME_{\mathcal{D}}^2(\mu,\nu;g) = GME_{\mathcal{D}}^2(\bar{\mu},\bar{\nu}),
$
where $\bar{\mu}=f\sharp \mu$ and $\bar{\nu}=f\sharp \nu$ with $f(x) = (x,g(x))$ for $g:\mathbb{R}^d \to \mathbb{R}$ is a non-linear function. For any $g$, $AGME_{\mathcal{D}}^2(\mu,\nu;g)$ is still a valid distance since $f$ is injective. In the diffusion model setting, we train $g_\gamma$ (parameterized by $\gamma$) as a time-conditional discriminator as in~\citep[Equation (5)]{xiao2021tackling}. We refer the reader to Algorithm~\ref{alg:DD} in Appendix~\ref{sec:algorithm} for a detailed algorithm, and to Appendix~\ref{sec:add_exps} for more detailed experimental settings.

\begin{table}[!t]
    \centering
    \caption{\footnotesize{Results for unconditional generation on CIFAR-10.}}
    \vskip 0.05in
    \scalebox{0.8}{
    \begin{tabular}{lccc}
    \toprule
    Model & FID$\downarrow$ &NFE$\downarrow$& G-Time (s $\downarrow$)
    \\
    \midrule
    DDPM~\cite{ho2020denoising} &3.17&1000&80.5\\
    Score SDE (VE)~\cite{song2020score}& 2.20&2000&423.2\\
    Score SDE (VP)~\cite{song2020score}& 2.41&2000&421.5\\
        DDIM~\cite{song2020denoising}&4.67&50&4.01\\
    Prob-Flow (VP)~\cite{song2020score}&3.08&140&50.9\\
    LSGM~\cite{vahdat2021score}&2.10&147&44.5 \\
    \midrule
    DDGAN~\cite{xiao2021tackling} &3.64& 4&0.21 \\
    SW-DD (ours) & 2.90 &4&0.21 \\
    Max-SW-DD (ours) & 2.99 &4&0.21\\
    DSW-DD (ours) & 2.88&4&0.21\\
    EBSW-DD (ours) & 2.87&4&0.21\\
    RPSW-DD (ours) &2.82&4&0.21\\
    IWRPSW-DD (ours) &\textbf{2.70}&4&0.21\\
    \bottomrule
    \end{tabular}
    }
    \label{tab:fid}
    \vskip -0.2in
\end{table}

\textbf{Results.} We follow the setting in~\cite{xiao2021tackling} for diffusion models on CIFAR10~\cite{krizhevsky2009learning} with $N=1800$ epochs. We set $L=10^4$ for SW, DSW, EBSW, RPSW, and IWRPSW. We set $T \in \{2,5,10\}$ for DSW and $T \in \{100,1000\}$ for Max-SW. It is worthing noting that the training time of Max-SW and DSW is more than two times the time for SW, EBSW, RPSW, and IWRPSW, hence, we cannot increase the value of $T$ further.  We adjust concentration parameter $\kappa$ of the PS distribution for each epoch $t$ as $\kappa(t)=(\kappa_0-1) \left(\frac{N-t-1}{N-1}\right)^{10} +1$ with $N=1800$ and $\kappa_0 \in \{100,50\}$. We report the FID scores~\cite{heusel2017gans}, the number of function evaluations, and the generation time in Table~\ref{tab:fid}. Using an implicit denoising model with only $4$ time steps, our SW diffusion models have a much faster generation time compared to existing diffusion models such as DDPM, Score SDE, DDIM, Probability Flow, and LSGM. Compared to DDGAN with the same $4$ time steps, our SW diffusion models lead to a lower FID score and the same generation time. Among SW variants, the proposed RPSW and IWRPSW achieve the lowest scores of 2.82 and 2.70 in turn. We show some random generated images from DDGAN, RPSW-DD, and IWRPSW-DD in Figure~\ref{fig:cifar10} as qualitative results.
Overall, AGME distance with RPSW and IWRPSW kernels could be a potential effective loss for improving training implicit diffusion models\footnote{Code for this paper is published at  \url{https://github.com/khainb/RPSW}.}.
\section{Conclusion}
\label{sec:conclusion}
We have presented the new construction of discriminative projecting directions for SW i.e., the random-path projecting direction. From that, we derive the random-path slicing distribution and two optimization-free variants of SW, named random-path sliced Wasserstein (RPSW) and importance weighted random-path sliced Wasserstein (IWRPSW). We then discuss the topological properties, statistical properties, and computational properties of the proposed RPSW and IWRPSW. Finally, we show the favorable performance of RPSW and IWRPSW in gradient flow and training denoising diffusion models. In the future, we will extend the random-path directions to probability measures that have supported on the hyper-sphere~\cite{bonet2022spherical}, and hyperbolic manifolds~\cite{bonet2023hyperbolic}. In such cases, the random path will not be straight lines but curves.

\clearpage

\section*{Impact Statements}
The paper proposes a new variant of sliced Wasserstein to compare two probability measures. Given the numerous applications in machine learning, such as generative modeling, clustering, classification, domain adaptation, and more, the proposed random-path variants of sliced Wasserstein distance could enhance the performance of downstream applications in both quality and computation, as demonstrated in the paper.

\bibliography{example_paper}
\bibliographystyle{icml2024}
\clearpage

%%%%%%%%%%%%%%%%%%%%%%%%%%%%%%%%%%%%%%%%%%%%%%%%%%%%%%%%%%%%%%%%%%%%%%%%%%%%%%%
%%%%%%%%%%%%%%%%%%%%%%%%%%%%%%%%%%%%%%%%%%%%%%%%%%%%%%%%%%%%%%%%%%%%%%%%%%%%%%%
% APPENDIX
%%%%%%%%%%%%%%%%%%%%%%%%%%%%%%%%%%%%%%%%%%%%%%%%%%%%%%%%%%%%%%%%%%%%%%%%%%%%%%%
%%%%%%%%%%%%%%%%%%%%%%%%%%%%%%%%%%%%%%%%%%%%%%%%%%%%%%%%%%%%%%%%%%%%%%%%%%%%%%%
\newpage
\appendix
\onecolumn

\begin{center}
{\bf{\Large{Supplement to ``Sliced Wasserstein with Random-Path Projecting Projections"}}}
\end{center}
First, we present skipped proofs in the main text in Appendix~\ref{sec:proof}. We then discuss some related works in~\ref{sec:relatedwork}. After that, we provide some additional experimental results on gradient estimators in Appendix~\ref{sec:add_exps}.  Finally, we report the computational infrastructure in Appendix~\ref{sec:infra}.
\section{Proofs}
\label{sec:proof}

\subsection{Proof of Theorem~\ref{theo:metricity}}
\label{subsec:proof:theo:metricity}

\textbf{Non-negativity.} Since the Wasserstein distance is non-negative, we have $W_p(\theta \sharp \mu,\theta \sharp \nu)\geq 0$ for all $\theta \in \mathbb{S}^{d-1}$. Therefore, $\mathbb{E}_{\theta \sim \sigma_{RP}(\theta;\mu,\nu,\sigma_\kappa)}[W_p(\theta \sharp \mu,\theta \sharp \nu)] \geq 0$ which leads to $\text{RPSW}_p(\mu,\nu;\sigma_\kappa)\geq 0$. Similarly, we have $\mathbb{E}_{\theta_1,\ldots, \theta_L\sim \sigma_{RP}(\theta;\mu,\nu,\sigma_\kappa)} \left[\sum_{l=1}^L W_p^p(\theta_l \sharp \mu, \theta_l \sharp \nu) \frac{f(W_p^p(\theta_l \sharp \mu, \theta_l \sharp \nu))}{\sum_{j=1}^L f(W_p^p(\theta_j \sharp \mu, \theta_j \sharp \nu))} \right] \geq 0 $ which implies $\text{IWRPSW}_p(\mu,\nu;\sigma_\kappa,L)\geq 0$

\textbf{Symmetry.} From the definition of RPSW from Remark~\ref{remark:RPSW}, and $f_{\sigma_\kappa(P_{\mathbb{S}^{d-1}}(x-y))} (\theta)$ as the density function of $\sigma_\kappa(P_{\mathbb{S}^{d-1}}(x-y))$ we have:
\begin{align*}
    \text{RPSW}^p_p(\mu,\nu;\sigma_\kappa) &=\mathbb{E}_{X \sim \mu,Y\sim \nu} \mathbb{E}_{\theta \sim\sigma_\kappa(\theta;P_{\mathbb{S}^{d-1}}(X-Y)) }[W_p^p(\theta \sharp \mu, \theta \sharp \nu)] \\
    &=\int_{\mathbb{R}^d} \int_{\mathbb{R}^d}  \int_{\mathbb{S}^{d-1}} W_p^p(\theta \sharp \mu, \theta \sharp \nu) f_{\sigma_\kappa(P_{\mathbb{S}^{d-1}}(x-y))} (\theta)  d\theta d \mu(x) d \nu(y) \\
    &=\int_{\mathbb{R}^d} \int_{\mathbb{R}^d}  \int_{\mathbb{S}^{d-1}} W_p^p(-\theta \sharp \mu, -\theta \sharp \nu) f_{\sigma_\kappa(P_{\mathbb{S}^{d-1}}(x-y))} (-\theta)  d\theta d \mu(x) d \nu(y) \\
&=\int_{\mathbb{R}^d} \int_{\mathbb{R}^d}  \int_{\mathbb{S}^{d-1}} W_p^p(\theta \sharp \mu, \theta \sharp \nu) f_{\sigma_\kappa(P_{\mathbb{S}^{d-1}}(y-x))} (\theta)  d\theta d \mu(x) d \nu(y)\\
&=\int_{\mathbb{R}^d} \int_{\mathbb{R}^d}  \int_{\mathbb{S}^{d-1}} W_p^p(\theta \sharp \nu, \theta \sharp \mu) f_{\sigma_\kappa(P_{\mathbb{S}^{d-1}}(y-x))} (\theta)  d\theta  d \nu(y)d \mu(x)\\
    &=\mathbb{E}_{Y\sim \nu,X \sim \mu} \mathbb{E}_{\theta \sim\sigma_\kappa(\theta;P_{\mathbb{S}^{d-1}}(Y-X)) }[W_p^p(\theta \sharp \nu, \theta \sharp \mu)]  = \text{RPSW}^p_p(\nu,\mu;\sigma_\kappa),
\end{align*}
where we use the reflection property of $\mathbb{S}^{d-1}$,  the property:
\begin{align*}
    W_p^p(\theta \sharp \mu,\theta \sharp \nu)&= \inf_{\pi \in \Pi(\mu,\nu)} \int |\theta^\top(x - y)|^p d\pi(x,y) \\
    &=\inf_{\pi \in \Pi(\mu,\nu)} \int |-\theta^\top(x - y)|^p d\pi(x,y)  = W_p^p(-\theta \sharp \mu,-\theta \sharp \nu),
\end{align*}
and $$f_{\sigma_\kappa(\theta;P_{\mathbb{S}^{d-1}}(x-y))}(\theta) = f_{\sigma_\kappa(\theta;P_{\mathbb{S}^{d-1}}(y-x))}(-\theta)$$ which holds for both the vMF density $f_{\sigma_\kappa(\theta;P_{\mathbb{S}^{d-1}}(x-y)}(\theta) \propto \exp\left(\kappa \frac{(x-y)}{\|x-y\|_2}^\top \theta\right)$ and the PS density $f_{\sigma_\kappa(\theta;P_{\mathbb{S}^{d-1}}(x-y)}(\theta) \propto \left(1+\frac{(x-y)}{\|x-y\|_2}^\top \theta\right)^\kappa $. Similarly, we have $
    \text{IWRPSW}^p_p(\mu,\nu;\sigma_\kappa,L,f) =\text{IWRPSW}^p_p(\nu,\mu;\sigma_\kappa,L,f).$
    
\textbf{Identity.}  We need to show that $\text{RPSW}_{p}(\mu,\nu;\sigma_k) = 0 $ if and only if $ \mu =  \nu$. For the forward direction, since $W_p(\theta\sharp \mu,\theta \sharp \nu)=0$ when $\mu=\nu$, we obtain directly $\mu = \nu$ implies $\text{RPSW}_{p}(\mu,\nu;\sigma_k) = 0 $. For the reverse direction, we use the same proof technique in~\cite{bonnotte2013unidimensional}. If $\text{RPSW}_{p}(\mu,\nu;\sigma_k) = 0$, we have $\int_{\mathbb{S}^{d-1}}\text{W}_p\left(\theta {\sharp} \mu, \theta \sharp \nu\right) \mathrm{d} \sigma_{RP}(\theta;\mu,\nu,\sigma_k)=0$. Hence, we have $W_p(\theta \sharp \mu,\theta \sharp \nu) =0 $ for $\sigma_{RP}(\theta;\mu,\nu,\sigma_k)$-almost surely $\theta \in \mathbb{S}^{d-1}$. Since $\sigma_{RP}(\theta;\mu,\nu,\sigma_k)$ is continuous due to the continuity of $\sigma_k$, we have  $W_p(\theta \sharp \mu,\theta \sharp \nu) =0 $ for all $\theta \in \mathbb{S}^{d-1}$. Therefore, we obtain $\theta \sharp \mu =  \theta \sharp \nu$ for $\sigma_{\mu,\nu}(\theta;f,p)$-a.e  $\theta \in \mathbb{S}^{d-1}$ due to the identity of Wasserstein distance. For any $t \in \mathbb{R}$ and $\theta \in \mathbb{S}^{d-1}$, we have:
    \begin{align*}
       \mathcal{F}[\mu](t\theta) &= \int_{\mathbb{R}^{d}} e^{-it\langle \theta,x \rangle} d\mu(x)=  \int_{\mathbb{R}}e^{-itz} d \theta \sharp \mu(z) = \mathcal{F}[\theta \sharp \mu](t) \\
       &=\mathcal{F}[\theta \sharp \nu](t) =\int_{\mathbb{R}}e^{-itz} d \theta \sharp \nu(z) =\int_{\mathbb{R}^{d}} e^{-it\langle \theta,x \rangle} d\nu(x)=\mathcal{F}[\nu](t\theta), 
    \end{align*}
    where $\mathcal{F}[\gamma](w) = \int_{\mathbb{R}^{d'}} e^{-i \langle w,x\rangle} d \gamma(x)$ denotes the Fourier transform of $\gamma \in \mathcal{P}(\mathbb{R}^{d'})$. The above transformation is also known as the slice projection theorem. By the injectivity of the Fourier transform. 

For IWRPSW,  when $\mu=\nu$, we have $W_p(\theta_l \sharp \mu,\theta_l \sharp \nu) = 0$ for any $\theta_1,\ldots,\theta_L \in \mathbb{S}^{d-1}$. Therefore, we have $\sum_{l=1}^L W_p^p(\theta_l \sharp \mu, \theta_l \sharp \nu) \frac{f(W_p^p(\theta_l \sharp \mu, \theta_l \sharp \nu))}{\sum_{j=1}^L W_p^p(\theta_j \sharp \mu, \theta_j \sharp \nu)}=0$ for any $\theta_1,\ldots,\theta_L \in \mathbb{S}^{d-1}$ which implies $\text{IWRPSW}_p(\mu,\nu;\sigma_\kappa,L)= \mathbb{E} \left[\sum_{l=1}^L W_p^p(\theta_l \sharp \mu, \theta_l \sharp \nu) \frac{f(W_p^p(\theta_l \sharp \mu, \theta_l \sharp \nu))}{\sum_{j=1}^L W_p^p(\theta_j \sharp \mu, \theta_j \sharp \nu)}\right]= 0$.  In the reverse direction, when  $\text{IWRPSW}_p(\mu,\nu;\sigma_\kappa,L)=0$, it means that 
we have $\sum_{l=1}^L W_p^p(\theta_l \sharp \mu, \theta_l \sharp \nu) \frac{f(W_p^p(\theta_l \sharp \mu, \theta_l \sharp \nu))}{\sum_{j=1}^L W_p^p(\theta_j \sharp \mu, \theta_j \sharp \nu)}=0$ 
for any $\theta_1,\ldots,\theta_L \in \mathbb{S}^{d-1}$. Since $f(W_p^p(\theta_l \sharp \mu, \theta_l \sharp \nu))>0$ for any $\theta_j$, it implies that $W_p^p(\theta_l \sharp \mu, \theta_l \sharp \nu)=0$ for all $\theta_l \in \mathbb{S}^{d-1}$. With similar arguments to the proof of RPSW, we obtain $\mu=\nu$ which completes the proof.

\textbf{Quasi-Triangle Inequality.}  Given three probability measures $\mu_1,\mu_2,\mu_3 \in \mathcal{P}_p(\mathbb{R}^d)$ we have:
\begin{align*}
    \text{RPSW}_p(\mu_1,\mu_2;\sigma_\kappa) &=\left(\mathbb{E}_{\theta \sim \sigma_{RP}(\theta;\mu_1,\mu_2,\sigma_\kappa)}[W_p^p(\theta \sharp \mu_1, \theta \sharp \mu_2)]\right)^{\frac{1}{p}} \\
    &\leq \left(\mathbb{E}_{\theta \sim \sigma_{RP}(\theta;\mu_1,\mu_2,\sigma_\kappa)}[(W_p(\theta \sharp \mu_1, \theta \sharp \mu_3)+ W_p(\theta \sharp \mu_3, \theta \sharp \mu_2))^p]\right)^{\frac{1}{p}} \\
    &\leq \left(\mathbb{E}_{\theta \sim \sigma_{RP}(\theta;\mu_1,\mu_2,\sigma_\kappa)}[W_p^p(\theta \sharp \mu_1, \theta \sharp \mu_3)]\right)^{\frac{1}{p}} +\left(\mathbb{E}_{\theta \sim \sigma_{RP}(\theta;\mu_1,\mu_2,\sigma_\kappa)}[W_p^p(\theta \sharp \mu_3, \theta \sharp \mu_2)]\right)^{\frac{1}{p}} \\
    &= \text{RPSW}_{p}(\mu_1,\mu_3;\sigma_\kappa,\mu_1,\mu_2) + \text{RPSW}_{p}(\mu_3,\mu_2;\sigma_\kappa,\mu_1,\mu_2), 
\end{align*}
where the first inequality is due to the triangle inequality of Wasserstein distance and the second inequality is due to the Minkowski inequality. We complete the proof here.
\subsection{Proof of Proposition~\ref{prop:connection}}
\label{subsec:proof:prop:connection}

(i) To prove that $\text{RPSW}_{p}(\mu,\nu;\sigma_\kappa)\leq \text{IWRPSW}_{p}(\mu,\nu;\sigma_\kappa, L)$, we introduce the following lemma which had been proved in~\cite{nguyen2023energy}. Here, we provide the proof for completeness.

\begin{lemma}
    \label{lemma:inequality} For any $L\geq 1$, $0\leq a_{1} \leq a_{2} \leq \ldots \leq a_{L}$ and $0< b_{1} \leq b_{2} \leq \ldots \leq b_{L}$, we have:
    \begin{align}
        \frac{1}{L} (\sum_{i = 1}^{L} a_{i}) (\sum_{i = 1}^{L} b_{i}) \leq \sum_{i = 1}^{L} a_{i} b_{i}. \label{eq:key_inequality_2}
    \end{align}
\end{lemma}
\begin{proof}
    We prove Lemma~\ref{lemma:inequality} via an induction argument. For $L=1$, it is clear that $a_ib_i = a_ib_i$. Now, we assume that the inequality holds for $L$ i.e., $\frac{1}{L} (\sum_{i = 1}^{L} a_{i}) (\sum_{i = 1}^{L} b_{i}) \leq \sum_{i = 1}^{L} a_{i} b_{i}$ or  $
    (\sum_{i = 1}^{L} a_{i}) (\sum_{i = 1}^{L} b_{i}) \leq L \sum_{i = 1}^{L} a_{i} b_{i}.$
Now, we want to show that the inequality holds for $L+1$ i.e., $ (\sum_{i = 1}^{L+1} a_{i}) (\sum_{i = 1}^{L} b_{i}) \leq (L+1) \sum_{i = 1}^{L+1} a_{i} b_{i}.$ First, we have:
\begin{align*}
    (\sum_{i = 1}^{L + 1} a_{i}) (\sum_{i = 1}^{L + 1} b_{i}) &= (\sum_{i = 1}^{L } a_{i}) (\sum_{i = 1}^{L} b_{i})+ (\sum_{i = 1}^{L} a_{i}) b_{L + 1} + (\sum_{i = 1}^{L} b_{i}) a_{L + 1} + a_{L + 1} b_{L + 1} \\
    &\leq L \sum_{i = 1}^{L} a_{i} b_{i} + (\sum_{i = 1}^{L} a_{i}) b_{L + 1} + (\sum_{i = 1}^{L} b_{i}) a_{L + 1} + a_{L + 1} b_{L + 1}.
\end{align*}
By rearrangement inequality, we have $a_{L + 1} b_{L + 1} + a_{i} b_{i} \geq a_{L + 1} b_{i} + b_{L + 1} a_{i}$ for all $1 \leq i \leq L$. By taking the sum of these inequalities over $i$ from $1$ to $L$, we obtain:
\begin{align*}
    (\sum_{i = 1}^{L} a_{i}) b_{L + 1} + (\sum_{i = 1}^{L} b_{i}) a_{L + 1} \leq \sum_{i = 1}^{L} a_{i} b_{i} + L a_{L + 1} b_{L + 1}. \label{eq:key_inequality_3}
\end{align*}

Therefore, we have

\begin{align*}
    (\sum_{i = 1}^{L + 1} a_{i}) (\sum_{i = 1}^{L + 1} b_{i}) &\leq L \sum_{i = 1}^{L} a_{i} b_{i} + (\sum_{i = 1}^{L} a_{i}) b_{L + 1} + (\sum_{i = 1}^{L} b_{i}) a_{L + 1} + a_{L + 1} b_{L + 1} \\
    &\leq L \sum_{i = 1}^{L} a_{i} b_{i} + \sum_{i = 1}^{L} a_{i} b_{i} + L a_{L + 1} b_{L + 1} + a_{L + 1} b_{L + 1}  \\
    &=(L + 1) (\sum_{i = 1}^{L + 1} a_{i} b_{i}),
\end{align*}
which completes the proof.
\end{proof}

From Lemma~\ref{lemma:inequality}, with $a_i = W_p^p(\theta_i\sharp \mu,\theta_j \sharp \nu)$ and $b_i = f(W_p^p(\theta_i\sharp \mu,\theta_j \sharp \nu))$, we have:
\begin{align*}
    \frac{1}{L} \sum_{i=1}^l W_p^p(\theta_i\sharp \mu,\theta_i \sharp \nu) \leq \sum_{i=1}^L W_p^p(\theta_i\sharp \mu,\theta_i \sharp \nu) \frac{f(W_p^p(\theta_i\sharp \mu,\theta_i \sharp \nu)) }{\sum_{j=1}^L f(W_p^p(\theta_j\sharp \mu,\theta_j \sharp \nu))}.
\end{align*}
Taking the expectation with respect to $\theta_1,\ldots,\theta_L \overset{i.i.d}{\sim} \sigma_{RP}(\theta;\mu,\nu,\sigma_\kappa)$, we obtain $\text{RPSW}_{p}(\mu,\nu;\sigma_\kappa)\leq \text{IWRPSW}_{p}(\mu,\nu;\sigma_\kappa, L)$.

Now to show that $\text{IWRPSW}_{p}(\mu,\nu;\sigma_\kappa, L)\leq \text{Max-SW}(\mu,\nu)$, we have  $\theta^\star =\text{argmax}_{\theta \in \mathbb{S}^{d-1}} W_p(\theta \sharp \mu,\theta \sharp \nu)$ since  $\mathbb{S}^{d-1}$ is compact and the function $\theta \to W_p(\theta \sharp \mu,\theta \sharp \nu)$ is continuous. From the definition of the IWRPSW, for any $L\geq 1, \sigma_\kappa \in \mathcal{P}(\mathbb{S}^{d-1})$ we have:
\begin{align*}
    \text{IWRPSW}_{p}(\mu,\nu;\sigma_\kappa, L) &=  \left(\mathbb{E}_{\theta_1,\ldots, \theta_L\sim \sigma_{RP}(\theta;\mu,\nu,\sigma_\kappa)}  \left[\sum_{l=1}^L W_p^p(\theta_l \sharp \mu, \theta_l \sharp \nu) \frac{f(W_p^p(\theta_l \sharp \mu, \theta_l \sharp \nu))}{\sum_{j=1}^L f(W_p^p(\theta_j \sharp \mu, \theta_j \sharp \nu))} \right]\right)^{\frac{1}{p}} \\
    & \leq \left(\mathbb{E}_{\theta \sim \sigma_{RP}(\theta;\mu,\nu,\sigma_k)} \left[ W_p^p \left(\theta^\star \sharp \mu, \theta^\star \sharp \nu \right)\right]\right)^{\frac{1}{p}} = W_p^p \left(\theta^\star \sharp \mu, \theta^\star \sharp \nu \right) = \text{Max-SW}_p(\mu,\nu).
\end{align*}
Furthermore,  by applying the Cauchy-Schwartz inequality, we have:
\begin{align*}
\text{Max-SW}_p(\mu,\nu) &= \left(\max _{\theta \in \mathbb{S}^{d-1}}\left(\inf _{\pi \in \Pi(\mu, \nu)} \int_{\mathbb{R}^d}\left|\theta^{\top} x-\theta^{\top} y\right|^p d \pi(x, y)\right) \right)^{\frac{1}{p}} \\
& \leq \left(\max _{\theta \in \mathbb{S}^{d-1}}\left(\inf _{\pi \in \Pi(\mu, \nu)} \int_{\mathbb{R}^d \times \mathbb{R}^d}\|\theta\|^p\|x-y\|^p d \pi(x, y)\right) \right)^{\frac{1}{p}}\\
&=\left(\inf _{\pi \in \Pi(\mu, \nu)} \int_{\mathbb{R}^d \times \mathbb{R}^d}\|\theta\|^p\|x-y\|^p d \pi(x, y) \right)^{\frac{1}{p}}\\
&=\left(\inf _{\pi \in \Pi(\mu, \nu)} \int_{\mathbb{R}^d \times \mathbb{R}^d}\|x-y\|^p d \pi(x, y)\right)^{\frac{1}{p}} \\
&\leq \left( \inf _{\pi \in \Pi(\mu, \nu)} \int_{\mathbb{R}^d \times \mathbb{R}^d}|x-y|^p d \pi(x, y)\right)^{\frac{1}{p}}\\
&=W_p(\mu, \nu),
\end{align*}
 we complete the proof.

(ii)  We first recall the density of The von-Mises Fisher distribution and the Power spherical distribution. In particular, we have $vMF(\theta;\epsilon,\kappa):= \frac{\kappa^{d/2-1}}{(2\pi)^{d/2} I_{d/2-1}(\kappa)} \exp(\kappa\epsilon^\top \theta)$ with $I_s$  denotes the modified Bessel function of the first kind, and the PS distribution $PS(\theta;\epsilon,\kappa)  =  \left( 2^{d+\kappa -1} \pi^{(d-1)/2} \frac{\Gamma( (d-1)/2+\kappa )}{\Gamma (d+\kappa -1)} \right)^{-1}(1+\epsilon^\top \theta)^\kappa $. We have:
\begin{align*}
    &\lim_{\kappa \to 0}\frac{\kappa^{d/2-1}}{(2\pi)^{d/2} I_{d/2-1}(\kappa)} \exp(\kappa\epsilon^\top \theta) \to C_1,\\
    &\lim_{\kappa \to 0}\left( 2^{d+\kappa -1} \pi^{(d-1)/2} \frac{\Gamma( (d-1)/2+\kappa )}{\Gamma (d+\kappa -1)} \right)^{-1}(1+\epsilon^\top \theta)^\kappa\to C_2
\end{align*}
for some constant $C_1$ and $C_2$ which do not depend on $\theta$, hence, the vMF distribution and the PS distribution converge to the uniform distribution when $\kappa \to 0$. Therefore, we have $vMF(\theta;\epsilon,\kappa) \to \mathcal{U}(\mathbb{S}^{d-1})$ and $PS(\theta;\epsilon,\kappa) \to \mathcal{U}(\mathbb{S}^{d-1})$.
Therefore, we have $\sigma_{RP}(\theta;\mu,\nu,\sigma_\kappa) \to \mathcal{U}(\mathbb{S}^{d-1})$.

Now, we need to show that $W_p^p(\theta\sharp \mu,\theta\sharp \nu)$ is bounded and continuous with respective to $\theta$. For the boundedness, we have:
\begin{align*}
    \text{W}_p^p (\theta \sharp \mu,\theta \sharp \nu) & = \inf_{\pi\in\Pi(\nu,\mu)}\int_{\mathbb R^{d}} | \theta^\top x - \theta^\top y |^p d\pi(x,y) \\
     & \leq \inf_{\pi\in\Pi(\nu,\mu)}\int_{\mathbb R^{d}} \| x - y \|^p d\pi(x,y) \\
     & = \text{W}_p^p (\mu,\nu) <\infty.
\end{align*}
For the continuity, let $(\theta_t)_{t\geq 1}$ be a sequence on $\mathbb S^{d-1}$ which converges to $\theta\in\mathbb S^{d-1}$ i.e., $\|\theta_t - \theta \| \to 0$ as $t \to \infty$, and a arbitrary measure $\mu \in \mathcal{P}_p(\mathbb{R}^d)$. Then we have:
\begin{align*}
    \text{W}_p (\theta \sharp \mu,\theta_t \sharp \mu) & = \left(\inf_{\pi\in\Pi(\mu,\mu)}\int_{\mathbb R^{d}} | \theta_t^\top x - \theta^\top y |^p d\pi(x,y)\right)^{1/p}\\
     & \leq \left(\int_{\mathbb R^d} | \theta_t^\top x - \theta^\top x |^p d\mu(x)\right)^{1/p} \\
     & \leq \left(\int_{\mathbb R^d} \| x \|^p \mu(dx)\right)^{1/p} \| \theta_t - \theta\| \to 0 \quad \text{as } t\to\infty,
\end{align*}
where $\left(\int_{\mathbb R^d} \| x \|^p \mu(dx)\right)^{1/p} < \infty$ since $\mu \in \mathcal{P}_p(\mathbb{R}^d)$, and the second inequality is due to the Cauchy-Schwartz inequality. 

Using the triangle inequality, we have:
\begin{align*}
    \left|\text{W}_p (\theta_t \sharp \mu,\theta_t \sharp \nu) - \text{W}_p (\theta \sharp \mu,\theta \sharp \nu)\right| & \leq \left|\text{W}_p (\theta_t \sharp \mu,\theta_t \sharp \nu) - \text{W}_p (\theta \sharp \mu,\theta_t \sharp \nu)\right| + \left|\text{W}_p (\theta \sharp \mu,\theta_t \sharp \nu) - \text{W}_p (\theta \sharp \mu,\theta \sharp \nu) \right| \\
     & \leq \text{W}_p (\theta \sharp \mu,\theta_t \sharp \mu) + \text{W}_p (\theta \sharp \nu,\theta_t \sharp \nu) \to 0 \quad \text{as } t\to \infty,
\end{align*}
hence, $\text{W}_p (\theta_t \sharp \mu,\theta_t \sharp \nu) \to \text{W}_p (\theta \sharp \mu,\theta \sharp \nu)$ as $t\to\infty$ which complete the proof of continuity.

From the boundedness, continuity, and  $\sigma_{RP}(\theta;\mu,\nu,\sigma_\kappa) \to \mathcal{U}(\mathbb{S}^{d-1})$, we have $$\text{RPSW}_{p}^p(\mu,\nu;\sigma_\kappa) = \mathbb{E}_{\theta \sim \sigma_{RP}(\theta;\mu,\nu,\sigma_\kappa)}[W_p(\theta \sharp \mu, \theta \sharp \nu)] \to \mathbb{E}_{\theta \sim \mathcal{U}(\mathbb{S}^{d-1})}[W_p(\theta \sharp \mu, \theta \sharp \nu)] = SW_p^p(\mu,\nu).$$ Applying the continuous mapping theorem for $x\to x^{1/p}$, we obtain $\lim_{\kappa \to 0} \text{RPSW}_{p}(\mu,\nu;\sigma_\kappa) \to SW_p(\mu,\nu)$.

(iii) Since we have proved that $W_p^p(\theta\sharp \mu,\theta\sharp \nu)$ is bounded and continuous with respective to $\theta$, we can show that  $\sum_{i=1}^L W_p^p(\theta_i\sharp \mu,\theta_i \sharp \nu) \frac{f(W_p^p(\theta_i\sharp \mu,\theta_i \sharp \nu)) }{\sum_{j=1}^L f(W_p^p(\theta_j\sharp \mu,\theta_j \sharp \nu))}$ are bounded and continuous with respect to $\theta_1,\ldots,\theta_L$. As $L \to \infty$, we have $\sum_{i=1}^L W_p^p(\theta_i\sharp \mu,\theta_i \sharp \nu) \frac{f(W_p^p(\theta_i\sharp \mu,\theta_i \sharp \nu)) }{\sum_{j=1}^L f(W_p^p(\theta_j\sharp \mu,\theta_j \sharp \nu))} \to \mathbb{E}_{\gamma \sim \sigma_{\mu,\nu,f}(\gamma)}[\gamma \sharp \mu,\gamma \sharp \nu] = \text{EBSW}_p^p(\mu,\nu;f)$. Applying the continuous mapping theorem for $x\to x^{1/p}$, we obtain $\lim_{L \to \infty} \text{IWRPSW}_{p}(\mu,\nu;\sigma_\kappa,L,f) \to \text{EBSW}_p(\mu,\nu;f)$. 

(iv)
For any $\mu$ and $\nu$, by the Holder inequality,  we have:
\begin{align*}
    \text{RPSW}_{p}(\mu,\nu;\sigma_\kappa) &=  \left(\mathbb{E}_{\theta \sim \sigma_{RP}(\theta;\mu,\nu,\sigma_\kappa)}[W_p^p(\theta \sharp \mu, \theta \sharp \nu)]\right)^{\frac{1}{p}} \\
    &\geq \mathbb{E}_{\theta \sim \sigma_{RP}(\theta;\mu,\nu,\sigma_\kappa)}[W_p(\theta \sharp \mu, \theta \sharp \nu)] \\
    &\geq \mathbb{E}_{\theta \sim \sigma_{RP}(\theta;\mu,\nu,\sigma_\kappa)}[W_1(\theta \sharp \mu, \theta \sharp \nu)] \\&=  \text{RPSW}_{1}(\mu,\nu;\sigma_\kappa),
\end{align*}
which completes the proof.

\subsection{Proof of Proposition~\ref{prop:sample_complexity}}
\label{subsec:proof:prop:sample_complexity}
The proof for this result follows from the proof of Proposition 3 in~\cite{nguyen2023energy}. We assume that $\mu$ has a compact set of support $\mathcal{X} \in \mathbb{R}^d$. 

From Proposition~\ref{prop:connection}, we have:
\begin{align*}
    \mathbb{E} [\text{RPSW}_{p} (\mu_{n},\mu;\sigma_\kappa)] \leq  \mathbb{E} [\text{IWRPSW}_{p} (\mu_{n},\mu;\sigma_\kappa,L)] \leq \mathbb{E} \left[\text{Max-SW}_p (\mu_{n},\mu) \right],
\end{align*}
for any $\sigma_\kappa \in \mathcal{P}(\mathbb{S}^{d-1})$ 
 and $L\geq 1$. Therefore, the proposition follows as long as we can demonstrate that $$\mathbb{E} [\text{Max-SW}_p (\mu_{n},\mu)] \leq C\sqrt{(d+1) \log n/n}$$ where $\mu_n = \frac{1}{n}\sum_{i=1}^n \delta_{X_i}$ with $X_1,\ldots,X_n \overset{i.i.d}{\sim} \mu$, and $C > 0$ is some universal constant and the outer expectation is taken with respect to $X_1,\ldots,X_n$. 

Using the closed-form of one-dimensional Wasserstein distance, we have:
\begin{align*}
    \text{Max-SW}_p (\mu_{n},\mu) & = \max_{\theta \in \mathbb{S}^{d-1}} \int_{0}^{1} |F_{n, \theta}^{-1}(z) - F_{\theta}^{-1}(z)|^{p} d z,
\end{align*}
where $F_{n,\theta}$ and $F_{\theta}$ as the cumulative distributions of $\theta \sharp \mu_{n}$ and $\theta\sharp \mu$. Since $W_p( (t \theta) \sharp \mu,(t \theta) \sharp \nu) = tW_p( \theta \sharp \mu, \theta \sharp \nu)$ for $t > 0$. We can rewrite Max-SW as:
\begin{align*}
    \text{Max-SW}_p^{p} (\mu_{n},\mu) &=
      \max_{\theta \in \mathbb{R}^{d}: \|\theta\| = 1} \int_{0}^{1} |F_{n, \theta}^{-1}(z) - F_{\theta}^{-1}(z)|^{p} d z \\
    & \leq \text{diam}(\mathcal{X}) \max_{x \in \mathbb{R}, \theta \in \mathbb{R}^{d}: \|\theta\| \leq 1} |F_{n, \theta}(x) - F_{\theta}(x)|^{p}\\
    &=\text{diam}(\mathcal{X}) \sup_{A \in \mathcal{A}} |\mu_{n}(A) - \mu(A)|,
\end{align*}
where $\mathcal{A}$ is the set of half-spaces $\{z \in \mathbb{R}^{d}: \theta^{\top} z \leq x\}$ for all $\theta \in \mathbb{R}^{d}$ such that $\|\theta\| \leq 1$. 
From VC inequality (Theorem 12.5 in~\cite{devroye2013probabilistic}), we have
$$
    \mathbb{P}\left(\sup_{A \in \mathcal{A}} |\mu_{n}(B) - \mu(A)| >t \right) \leq 8 S(\mathcal{A},n) e^{-nt^2 /32}.
$$ with $S(\mathcal{A},n)$ is the growth function. From the Sauer Lemma (Proposition 4.18 in ~\cite{wainwrighthigh}), the  growth function is upper bounded by $(n+1)^{VC(\mathcal{A})}$. Moreover, we can get $VC(\mathcal{A})= d+1$ from Example 4.21 in~\cite{wainwrighthigh}.

\vspace{0.5em}
\noindent
Let $8 S(\mathcal{A},n) e^{-nt^2 /32} \leq \delta$, we have $t^2 \geq \frac{32}{n} \log \left( \frac{8S(\mathcal{A},n)}{\delta}\right)$. Therefore, we obtain
\begin{align*}
    \mathbb{P}\left(\sup_{A \in \mathcal{B}} |\mu_{n}(A) - \mu(A)| \leq \sqrt{\frac{32}{n} \log \left( \frac{8S(\mathcal{A},n)}{\delta}\right)} \right) \geq 1-\delta,
\end{align*}
Using the Jensen inequality and the tail sum expectation for non-negative random variable, we have:
\begin{align*}
    &\mathbb{E}\left[\sup_{A \in \mathcal{A}} |\mu_{n}(A) - \mu(A)|\right] \\&\leq \sqrt{\mathbb{E}\left[\sup_{A \in \mathcal{A}} |\mu_{n}(A) - \mu(A)|\right]^2} = \sqrt{\int_{0}^\infty \mathbb{P}\left(\left(\sup_{A \in \mathcal{A}} |\mu_{n}(A) - \mu(A)| \right)^2>t \right)dt} \\
    &=\sqrt{\int_{0}^u \mathbb{P}\left(\left(\sup_{A \in \mathcal{A}} |\mu_{n}(A) - \mu(A)| \right)^2>t \right)dt + \int_{u}^\infty \mathbb{P}\left(\left(\sup_{A \in \mathcal{A}} |\mu_{n}(A) - \mu(A)| \right)^2>t \right)dt} \\
    &\leq \sqrt{\int_{0}^u 1dt + \int_{u}^\infty8 S(\mathcal{A},n) e^{-nt /32} dt} = \sqrt{u + 256 S(\mathcal{A},n)  \frac{e^{-nu/32}}{n}}.
\end{align*}
Since the inequality holds for any $u$, we search for the best $u$ that makes the inequality tight. Let $f(u) = u + 256 S(\mathcal{A},n)  \frac{e^{-nu/32}}{n}$, we have $f'(u) = 1+ 8S(\mathcal{A},n) e^{-nu/32}$. Setting $f'(u)=0$, we obtain the minima $u^\star = \frac{32 \log (8S(\mathcal{A},n))}{n}$. Plugging $u^\star$ in the inequality, we obtain:
\begin{align*}
    \mathbb{E}\left[\sup_{A \in \mathcal{A}} |\mu_{n}(A) - \mu(A)|\right] & \leq \sqrt{\frac{32 \log (8S(\mathcal{A},n))}{n} + 32 }\leq C \sqrt{\frac{(d+1)\log (n+1)}{n}},
\end{align*}
by using Sauer Lemma i.e., $S(\mathcal{A},n) \leq (n+1)^{VC(\mathcal{A})} \leq (n+1)^{d+1}$. Putting the above results together leads to
\begin{align*}
    \mathbb{E} [\text{Max-SW}_p (\mu_{n},\mu)] \leq C\sqrt{(d+1) \log n/n},
\end{align*}
where $C > 0$ is some universal constant. As a consequence, we obtain the conclusion of the proof.

\subsection{Proof of Proposition~\ref{proposition:MCerror}}
\label{subsec:proof:proposition:MCerror}

For any $p\geq 1$, $d \geq 1$, $\sigma_\kappa \in \mathcal{P}(\mathbb{S}^{d-1})$, and $\mu,\nu \in \mathcal{P}_p(\mathbb{R}^d)$, using the Holder’s inequality, we have:
\begin{align*}
    &  \mathbb{E} | \widehat{\text{RPSW}}_{p}^p(\mu,\nu;\sigma_\kappa) - \text{RPSW}_{p}^p (\mu,\nu;\sigma_\kappa)| \\
    &\leq \left(\mathbb{E} | \widehat{\text{RPSW}}_{p}^p(\mu,\nu;\sigma_\kappa ) - \text{RPSW}_{p}^p (\mu,\nu;\sigma_\kappa)|^2 \right)^{\frac{1}{2}} \\
    &= \left(\mathbb{E} \left( \frac{1}{L}\sum_{l=1}^L  \text{W}_p^p (\theta_{l} \sharp \mu,\theta_{l} \sharp \nu)  - \mathbb{E}_{\theta\sim \sigma_{RP}(\theta;\mu,\nu,\sigma_\kappa)}\left[  W_p^p \left(\theta \sharp \mu, \theta \sharp \nu \right)\right]\right)^2 \right)^{\frac{1}{2}} 
\end{align*}

Since $\mathbb{E}[\frac{1}{L}\sum_{l=1}^L  \text{W}_p^p (\theta_{l} \sharp \mu,\theta_{l} \sharp \nu)] = \frac{1}{L}\sum_{l=1}^L \mathbb{E}[\text{W}_p^p (\theta_{l} \sharp \mu,\theta_{l} \sharp \nu)] = \frac{1}{L}\sum_{l=1}^L \widehat{\text{RPSW}}_{p}^p(\mu,\nu;\sigma_\kappa)  = \widehat{\text{RPSW}}_{p}^p(\mu,\nu;\sigma_\kappa) $, we have:
\begin{align*}
    \mathbb{E} | \widehat{\text{RPSW}}_{p}^p(\mu,\nu;\sigma_\kappa) - \text{RPSW}_{p}^p (\mu,\nu;\sigma_\kappa)|&\leq  \left( Var_{\theta \sim \sigma_{RP}(\theta;\mu,\nu,\sigma_\kappa )}\left[ \frac{1}{L}  \sum_{l=1}^L  W_p^p \left(\theta_l \sharp \mu, \theta_l \sharp \nu \right)\right]\right)^{\frac{1}{2}} \\
    &= \frac{1}{\sqrt{L}} Var\left[ W_p^p \left(\theta \sharp \mu, \theta \sharp \nu \right)\right]^{\frac{1}{2}},
\end{align*}
since $\theta_1,\ldots,\theta_L \overset{i.i.d}{\sim} \sigma_{RP}(\theta;\mu,\nu,\sigma_\kappa )$, which completes the proof.

\section{Related Works}
\label{sec:relatedwork}

\textbf{Dynamic Optimal Transport.} Wasserstein distance can also be seen as the shortest curve between two input measures~\cite{Villani-09} which creates an optimal interpolation path between two measures~\cite{mccann1997convexity}. In this work, we work with static sliced Wasserstein distance where the random-path is only used to construct discriminative projecting directions.

\textbf{Sliced Wasserstein on Manifolds.} On manifolds where a straight line is generalized into a curse, the random-path definition should be different. For example, on the sphere, the shortest path between two points is the great circle. Therefore, it is natural to define the random-path as a random great circle with two endpoints that are randomly drawn from a coupling between two measures. This definition can lead to a natural extension for spherical Sliced Wasserstein~\cite{bonet2022spherical}. Moreover, we can also create a similar definition on hyperbolic manifolds for hyperbolic sliced Wasserstein~\cite{bonet2023hyperbolic} and on the manifold of positive definite matrix~\cite{bonet2023sliced}.

\textbf{Dependent Projecting Directions.} Markovian Sliced Wasserstein distance is introduced in~\cite{nguyen2023markovian}. It is a variant of SW that utilizes dependent projecting directions i.e., follows a Markovian structure joint distribution. The contribution of random-path slicing distribution is orthogonal to the mentioned approach since as can be used in either the prior distribution or the transition distribution in the Markovian process. In this work, we focus on slicing distribution selection variants of SW without dependency structure which can be computed very fast in parallel. 

\textbf{Properties of sliced Wasserstein losses.} In this work, we use sliced Wasserstein distance as a loss for gradient flows and generative modeling. We refer the reader to~\cite{tanguy2023convergence} for a more discussion about its properties i.e., regularity, gradient definition, and so on.

\textbf{Unbalanced Sliced Wasserstein.} Random-path projecting direction can be applied directly to SW variants with unbalanced transport e.g., sliced partial optimal transport~\cite{bonneel2019spot,bai2022sliced}, sliced unbalanced optimal transport~\cite{sejourne2023unbalanced}.

\textbf{Other potential applications.} Since RPSW and IWRPSW can be used as a replacement for SW, they can be applied in domain adaptation~\cite{lee2019sliced}, point-cloud applications~\cite{bai2023partial,nguyen2023quasi}, 3D shapes matching and deformation~\cite{le2024integrating,le2024diffeomorphic}, Bayesian inference~\cite{yi2021sliced,nadjahi2020approximate}, blue noise sampling~\cite{paulin2020sliced} and so on.

\section{Algorithms}
\label{sec:algorithm}
\textbf{Computational algorithms of RPSW and IWRPSW.} We present the pseudo-codes for computing RPSW and IWRPSW with Monte Carlo estimation in Algorithm~\ref{alg:RPSW} and Algorithm~\ref{alg:IWRPSW}.

\textbf{Algorithm for training denoising diffusion models with augmented mini-batch energy distance.} The detailed algorithm is given in Algorithm~\ref{alg:DD}.

\begin{algorithm}[!t]
\caption{Computational algorithm of RPSW}
\begin{algorithmic}
\label{alg:RPSW}
\STATE \textbf{Input:} Probability measures $\mu$ and $\nu$, $p\geq 1$, the number of projections $L$, $0<\kappa<\infty$
  
  \FOR{$l=1$ to $L$}
  \STATE Sample $X \sim \mu$, $Y \sim \nu$
  \STATE Sample $\theta_l \sim \sigma_{\kappa}\left(\theta;\frac{X-Y}{\|X-Y\|_2}\right)$
  \STATE Compute $v_l = \text{W}_p(\theta_l \sharp \mu,\theta_l \sharp  \nu)$
  % \STATE Compute $w_l = f(\text{W}_p(\theta_l \sharp \mu,\theta_l \sharp  \nu))$
  \ENDFOR
  \STATE Compute $\widehat{\text{RPSW}}_p(\mu,\nu;L,\sigma_\kappa) = \left(\frac{1}{L}\sum_{l=1}^L v_l \right)^{\frac{1}{p}}$
  
 \STATE \textbf{Return:} $\widehat{\text{RPSW}}_p(\mu,\nu;L,\sigma_\kappa)$
\end{algorithmic}
\end{algorithm}

\begin{algorithm}[!t]
\caption{Computational algorithm of the IWRPSW}
\begin{algorithmic}
\label{alg:IWRPSW}
\STATE \textbf{Input:} Probability measures $\mu$ and $\nu$, $p\geq 1$, the number of projections $L$,  $0<\kappa<\infty$, and the energy function $f$.
  
  \FOR{$l=1$ to $L$}
  \STATE Sample $X \sim \mu$, $Y \sim \nu$
  \STATE Sample $\theta_l \sim \sigma_{\kappa}\left(\theta;\frac{X-Y}{\|X-Y\|_2}\right)$
  \STATE Compute $v_l = \text{W}_p(\theta_l \sharp \mu,\theta_l \sharp  \nu)$
  \STATE Compute $w_l = f(\text{W}_p(\theta_l \sharp \mu,\theta_l \sharp  \nu))$
  \ENDFOR
  \STATE Compute $\widehat{\text{IWRPSW}}_p(\mu,\nu;\sigma_\kappa,L,f) = \left(\sum_{l=1}^L v_l \frac{w_l}{\sum_{i=1}^L w_i}\right)^{\frac{1}{p}}$
  
 \STATE \textbf{Return:} $\widehat{\text{IWRPSW}}_p(\mu,\nu;\sigma_\kappa,L,f) $
\end{algorithmic}
\end{algorithm}

\begin{algorithm}[!t]
\caption{Training denoising diffusion models with augmented mini-batch energy distance}
\begin{algorithmic}
\label{alg:DD}
\STATE \textbf{Input:} $q(x_{1:T}|x_0) = \prod_{t\geq 1}q(x_t|x_{t-1}),$ $p_\phi(x_{0:T}) =p(x_T) \prod_{t\geq 1} p_\phi(x_{t-1}|x_t),$ $p_\phi(x_{t-1}|x_t) = \int p(\epsilon) q(x_{t-1}|x_t,x_0 = G_\phi(x_t,\epsilon,t))$ with $p(\epsilon)\in \mathcal{P}(\mathbb{R}^z)$ and $G_\phi:\mathbb{R}^d\times \mathbb{R}^z \times \mathbb{R} \to \mathbb{R}^d$, $D_\gamma:\mathbb{R}^d\times \mathbb{R}^d \times \mathbb{R} \to [0,1]$, metric $\mathcal{D}$.
  
  \WHILE{$\phi$ not converged or not reaching the maximum epochs}
  \STATE Sample a mini-batch $x_{0,1},\ldots,x_{0,m} \sim q(x_0)$ with $m \geq 2$.
  \STATE Sample $t \sim \mathcal{U}({1,\ldots,T})$.
  \STATE Sample a mini-batch $x_{t-1,1} \sim q(x_{t-1}|x_{0,1}),\ldots,x_{t-1,m} \sim q(x_{t-1}|x_{0,m})$.
  \STATE Sample a mini-batch $x_{t,1} \sim q(x_{t}|x_{t-1,1}),\ldots,x_{t,m} \sim q(x_{t}|x_{t-1,m})$.
  \STATE Sample a mini-batch of noise $ \epsilon_1,\ldots,\epsilon_m \sim p(\epsilon)$.
  \STATE Set $y_{0,1},\ldots,y_{0,m}=G_\phi(x_{t,1},\epsilon_1,t),\ldots, G_\phi(x_{t,m},\epsilon_m,t)$.
  \STATE Sample a mini-batch $y_{t-1,1} \sim q(x_{t-1}|x_{t,1},y_{0,1}), \ldots, y_{t-1,m} \sim q(x_{t-1}|x_{t,m},y_{0,m})$.
  \STATE Update $\gamma$ with $\nabla_\gamma \left(\frac{1}{m}\sum_{i=1}^m -\log (D_\gamma(x_{t-1,i},x_{t,i},t)) \right)$.
  \STATE Update $\gamma$ with $\nabla_\gamma \left(\frac{1}{m}\sum_{i=1}^m -\log (1-D_\gamma(y_{t-1,i},x_{t,i},t))\right)$.
  \STATE Set $\bar{X}=((x_{t-1,1},D_\gamma(x_{t-1,1},x_{t,1},t)),\ldots,  (x_{t-1,m},D_\gamma(x_{t-1,m},x_{t,m},t)) )$.
  \STATE Set $\bar{Y}=((y_{t-1,1},D_\gamma(y_{t-1,1},x_{t,1},t)),\ldots,  (y_{t-1,m},D_\gamma(y_{t-1,m},x_{t,m},t)) )$.
  \STATE Set $\bar{X}_1,\bar{X}_2 = \bar{X}[:m/2],\bar{X}[m/2:]$, and $\bar{Y}_1,\bar{Y}_2 = \bar{Y}[:m/2],\bar{Y}[m/2:]$
  \STATE Update $\phi$ with $\nabla_\phi \left(2 \mathcal{D}^2(\bar{X},\bar{Y}) - \mathcal{D}^2(\bar{X}_1,\bar{X}_2) - \mathcal{D}^2(\bar{Y}_1,\bar{Y}_2)\right)$
  \ENDWHILE
  
 \STATE \textbf{Return:} $\phi$
\end{algorithmic}
\end{algorithm}

\section{Additional Experimental Details}
\label{sec:add_exps}

 \begin{figure*}[t]
\begin{center}
  \begin{tabular}{cc}
  \widgraph{0.5\textwidth}{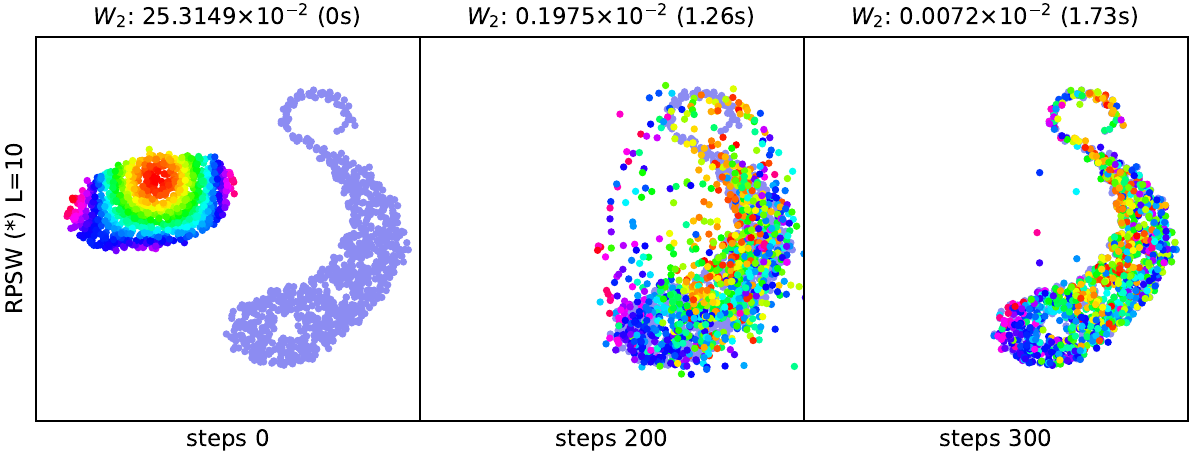}  & \widgraph{0.5\textwidth}{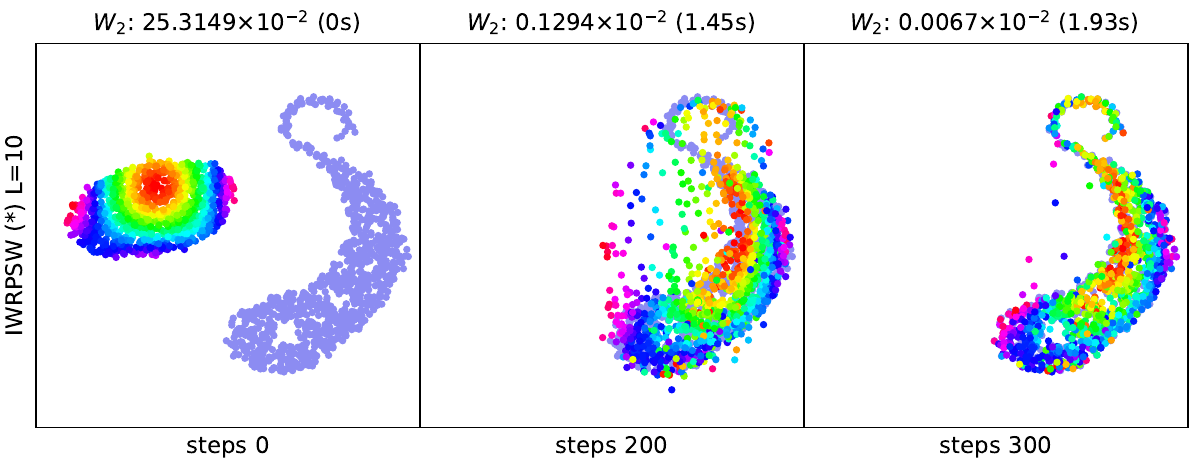} \\
  \end{tabular}
  \end{center}
  \vskip -0.2in
  \caption{
  {Results for gradient flows that are from the empirical distribution over the color points to the empirical distribution over S-shape points produced by different RPSW and IWRPSW with original gradient estimator.
}
} 
  \label{fig:gf_grad}
  \vskip -0.1in
\end{figure*}

\textbf{Gradient flows with different gradient estimators.} We run experiments on gradient flow with the original gradient estimators of RPSW and IWRPSW in Figure~\ref{fig:gf_grad}. We see that the original gradient estimators lead to faster convergence in terms of Wasserstein-2 distances. However, the inner topology between points, presented in colors, is destroyed especially for RPSW. The reason for this behavior is because of the independent coupling in constructing random paths. Therefore, we suggest using the simplified gradient estimator and treating the random-path slicing distribution as a separate distribution for projecting direction selection.

\textbf{Gradient flows on MNIST digits.} We present the full visualization of the gradient flows on MNIST in Figure~\ref{fig:gf_mnist_1000}.

\textbf{Neural network architecture for diffusion model.} Our generator follows the U-net structure in~\cite{xiao2021tackling}  which has 2 ResNet blocks per scale, 128 initial channels, channel multiplier for each scale: (1, 2, 2, 2), scale of attention block: 16, latent dimension 256, 3  latent mapping layers,   latent embedding dimension: 512. For the discriminator, the order of the layers are 1 × 1 conv2d, 128 $\to$ ResBlock, 128 $\to$ ResBlock down, 256
 $\to$ ResBlock down, 512 $\to$ ResBlock down, 512 $\to$ minibatch std layer $\to$ Global Sum Pooling $\to$ FC layer to scalar. For other settings, we refer the reader to~\citep[Section C]{xiao2021tackling}. 
 
\textbf{Hyper-parameters.}  We set  initial learning rate for discriminator to $10^{-4}$, initial learning rate for generator to $1.6 \times 10^{-4}$,  Adam optimizer with parameters $(0.5,0.9)$, EMA to $0.9999$, batch-size to $256$. For the learning rate scheduler, we use cosine learning rate decay.

 \begin{figure*}[t]
\begin{center}
  \begin{tabular}{cccc}
  \widgraph{0.19\textwidth}{images/SW_L1000_0.png}  & \widgraph{0.19\textwidth}{images/SW_L1000_99.png} & \widgraph{0.19\textwidth}{images/SW_L1000_999.png}  & \widgraph{0.19\textwidth}{images/SW_L1000_4999.png}\\
SW  (Step 0) & SW (Step 100) & SW (Step 1000) & SW (Step 5000)\\

  \widgraph{0.19\textwidth}{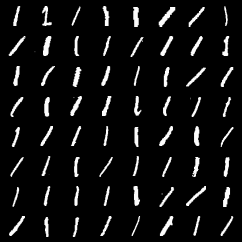}  & \widgraph{0.19\textwidth}{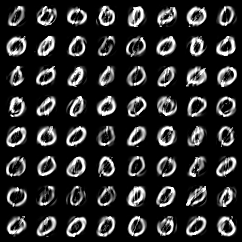} & \widgraph{0.19\textwidth}{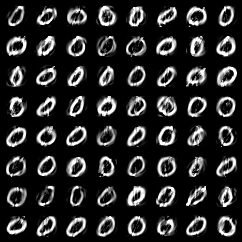}  & \widgraph{0.19\textwidth}{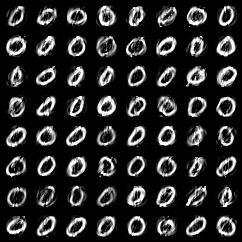}\\
Max-SW  (Step 0) & Max-SW (Step 100) & Max-SW (Step 1000) & Max-SW (Step 5000)\\
  \widgraph{0.19\textwidth}{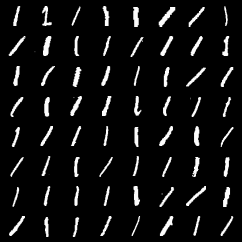}  & \widgraph{0.19\textwidth}{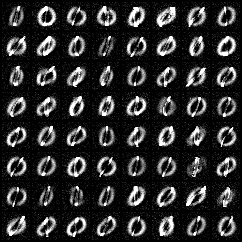} & \widgraph{0.19\textwidth}{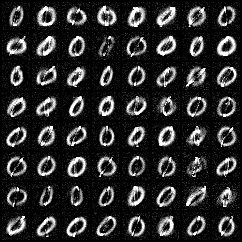}  & \widgraph{0.19\textwidth}{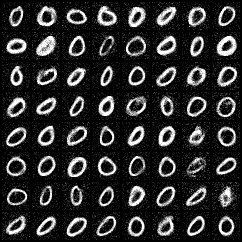}\\
DSW (Step 000) & DSW (Step 100) & DSW (Step 1000) & DSW (Step 5000)\\
\widgraph{0.19\textwidth}{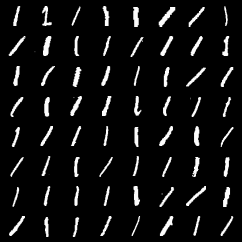}  & \widgraph{0.19\textwidth}{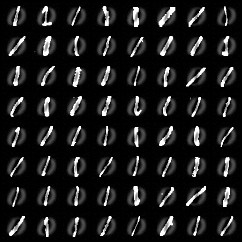} & \widgraph{0.19\textwidth}{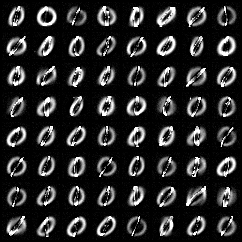}  & \widgraph{0.19\textwidth}{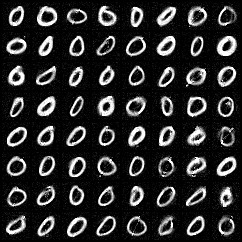}\\
EBSW (Step 0) & EBSW (Step 100) & EBSW  (Step 1000) & EBSW  (Step 5000)\\

\widgraph{0.19\textwidth}{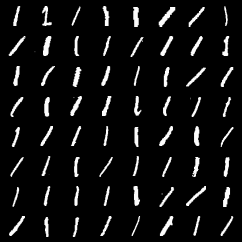}  & \widgraph{0.19\textwidth}{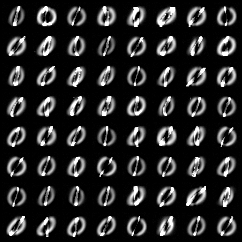} & \widgraph{0.19\textwidth}{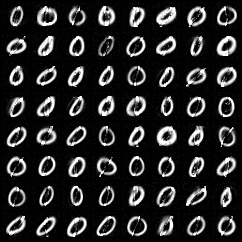}  & \widgraph{0.19\textwidth}{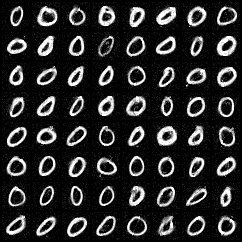}\\
RPSW (Step 0) & RPSW (Step 100) & RPSW (Step 1000) & RPSW (Step 5000)\\

\widgraph{0.19\textwidth}{images/IWRPSW_L1000_0.png}  & \widgraph{0.19\textwidth}{images/IWRPSW_L1000_499.png} & \widgraph{0.19\textwidth}{images/IWRPSW_L1000_999.png}  & \widgraph{0.19\textwidth}{images/IWRPSW_L1000_4999.png}\\
IWRPSW (Step 0) & IWRPSW  (Step 100) & IWRPSW  (Step 1000) & IWRPSW (Step 5000)\\

  \end{tabular}
  \end{center}
  \vskip -0.2in
  \caption{Gradient flows from MNIST digit 1 to MNIST digit 0
  {.
}
} 
  \label{fig:gf_mnist_1000}
  \vskip -0.1in
\end{figure*}

\section{Computational Infrastructure}
\label{sec:infra}

For the gradient flow experiments, we use a HP Omen 25L desktop for conducting experiments. For diffusion model experiments, we use a single NVIDIA A100 GPU.
\end{document}